\documentclass[10pt,journal,compsoc]{IEEEtran}
\usepackage[nocompress]{cite}
\usepackage[pdftex]{graphicx}
\usepackage{amsmath}
\usepackage{amssymb}
\usepackage{amsthm}
\usepackage{url}
\usepackage[colorlinks=true,citecolor=blue,urlcolor=blue]{hyperref}

\newcommand{\E}{\mathbb{E}}
\newcommand{\EE}[1]{\mathbb{E} \left[ {#1} \right]}
\newcommand{\eg}{e.g.~}
\newcommand{\ie}{i.e.~}
\newcommand{\etal}{et al.~}
\DeclareMathOperator{\argmax}{\arg\max}
\DeclareMathOperator{\argmin}{\arg\min}

\newtheorem{theorem}{Theorem}
\newtheorem{lemma}{Lemma}
\newtheorem{example}{Example}
\newtheorem{definition}{Definition}

\begin{document}
\title{The Perception-Distortion Tradeoff}

\author{Yochai~Blau
        and~Tomer~Michaeli%
\thanks{
Y. Blau and T. Michaeli are with the Technion -- Israel Institute of Technology, Haifa, Israel. E-mail: \{yochai@campus, tomer.m@ee\}.technion.ac.il \protect\\

\noindent This is an extended version of a paper published in the Proceedings of the 2018 IEEE Conference on Computer Vision and Pattern Recognition \cite{blau2018perception}. \protect\\
\url{https://ieeexplore.ieee.org/abstract/document/8578750} \protect\\
	
\noindent \scriptsize{\textcopyright \, 2020 IEEE. Personal use of this material is permitted. Permission from IEEE must be obtained for all other uses, in any current or future media, including reprinting/republishing this material for advertising or promotional purposes, creating new collective works, for resale or redistribution to servers or lists, or reuse of any copyrighted component of this work in other works.}}
}

\IEEEtitleabstractindextext{%
\begin{abstract}

Image restoration algorithms are typically evaluated by some distortion measure (\eg PSNR, SSIM, IFC, VIF) or by human opinion scores that quantify perceived perceptual quality. In this paper, we prove mathematically that distortion and perceptual quality are at odds with each other. Specifically, we study the optimal probability for correctly discriminating the outputs of an image restoration algorithm from real images. We show that as the mean distortion decreases, this probability must increase (indicating worse perceptual quality). As opposed to the common belief, this result holds true for any distortion measure, and is not only a problem of the PSNR or SSIM criteria. We also show that generative-adversarial-nets (GANs) provide a principled way to approach the perception-distortion bound. This constitutes theoretical support to their observed success in low-level vision tasks. Based on our analysis, we propose a new methodology for evaluating image restoration methods, and use it to perform an extensive comparison between recent super-resolution algorithms.

\end{abstract}
}

\maketitle
\IEEEraisesectionheading{\section{Introduction}\label{sec:introduction}}


\IEEEPARstart{T}{he} last decades have seen continuous progress in image restoration algorithms (\eg for denoising, deblurring, super-resolution) both in visual quality and in distortion measures like peak signal-to-noise ratio (PSNR) and structural similarity index (SSIM) \cite{wang2004image}. However, in recent years, it seems that the improvement in reconstruction accuracy is not always accompanied by an improvement in visual quality. In fact, and perhaps counter-intuitively, algorithms that are superior in terms of perceptual quality, are often inferior in terms of \eg PSNR and SSIM \cite{ledig2016photo,johnson2016perceptual,dahl2017pixel,sajjadi2017enhancenet,yeh2017semantic,yang2014single,shaham2019singan}. This phenomenon is commonly interpreted as a shortcoming of the existing distortion measures \cite{wang2009mean}, which fuels a constant search for alternative ``more perceptual'' criteria.

In this paper, we offer a complementary explanation for the apparent tradeoff between perceptual quality and distortion measures. Specifically, we prove that there exists a region in the perception-distortion plane, which cannot be attained regardless of the algorithmic scheme (see Fig.~\ref{fig:subObjPlane}). Furthermore, the boundary of this region is monotone. Therefore, in its proximity, it is only possible to improve \emph{either perceptual quality or distortion}, one at the expense of the other. The perception-distortion tradeoff exists for \emph{all distortion measures}, and is not only a problem of the mean-square error (MSE) or SSIM criteria. 

Let us clarify the difference between distortion and perceptual quality. The goal in image restoration is to estimate an image $x$ from its degraded version $y$ (\eg noisy, blurry, etc.). Distortion refers to the dissimilarity between the reconstructed image $\hat{x}$ and the original image $x$. Perceptual quality, on the other hand, refers only to the visual quality of $\hat{x}$, regardless of its similarity to $x$. Namely, it is the extent to which $\hat{x}$ looks like a valid natural image. An increasingly popular way of measuring perceptual quality is by using real-vs.\@-fake user studies, which examine the ability of human observers to tell whether $\hat{x}$ is real or the output of an algorithm~\cite{isola2016image,zhang2016colorful,salimans2016improved,denton2015deep,dahl2017pixel,iizuka2016let,zhang2017real,guadarrama2017pixcolor} (similarly to the idea underlying generative adversarial nets~\cite{goodfellow2014generative}). Therefore, perceptual quality can be defined as the best possible probability of success in such discrimination experiments, which as we show, is proportional to the distance between the distribution of reconstructed images and that of natural images.

\begin{figure}[]
	\begin{center}
		\includegraphics[width=0.87\linewidth]{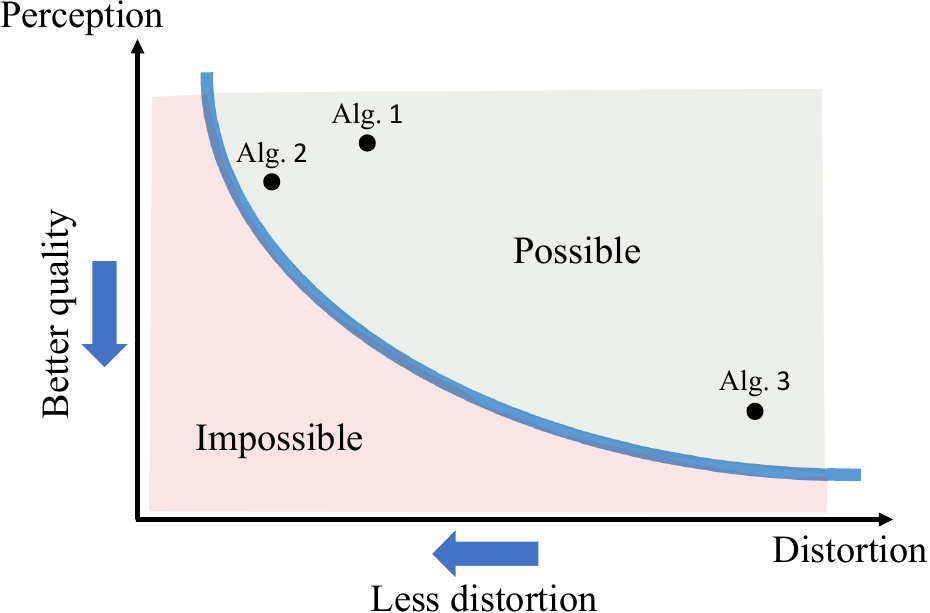}
	\end{center}
	\caption{\textbf{The perception-distortion tradeoff.} Image restoration algorithms can be characterized by their average distortion and by the perceptual quality of the images they produce. We show that there exists a region in the perception-distortion plane which cannot be attained, regardless of the algorithmic scheme. When in proximity of this unattainable region, an algorithm can be potentially improved only in terms of its distortion \emph{or} in terms of its perceptual quality, one at the expense of the other.}
	\label{fig:subObjPlane}
\end{figure}

Based on these definitions of perception and distortion, we follow the logic of rate-distortion theory~\cite{cover2012elements}. That is, we seek to characterize the behavior of the best attainable perceptual quality (minimal deviation from natural image statistics) as a function of the maximal allowable average distortion, for any estimator. This perception-distortion function (wide curve in Fig.~\ref{fig:subObjPlane}) separates between the attainable and unattainable regions in the perception-distortion plane and thus describes the fundamental tradeoff between perception and distortion. Our analysis shows that algorithms cannot be simultaneously very accurate \emph{and} produce images that fool observers to believe they are real, no matter what measure is used to quantify accuracy. This tradeoff implies that optimizing distortion measures can be not only ineffective, but also potentially damaging in terms of visual quality. This has been empirically observed \eg in \cite{ledig2016photo,johnson2016perceptual,sajjadi2017enhancenet,yeh2017semantic,dahl2017pixel}, but was never established theoretically.

From the standpoint of algorithm design, we show that generative adversarial nets (GANs) provide a principled way to approach the perception-distortion bound. This gives theoretical support to the growing empirical evidence of the advantages of GANs in image restoration \cite{ledig2016photo,sajjadi2017enhancenet,pathak2016context,yeh2017semantic,rippel2017real,isola2016image,zhu2017unpaired}.

The perception-distortion tradeoff has major implications on low-level vision. In certain applications, reconstruction accuracy is of key importance (\eg medical imaging). In others, perceptual quality may be preferred. The impossibility of simultaneously achieving both goals calls for a new way for evaluating algorithms: By placing them on the perception-distortion plane. We use this new methodology to conduct an extensive comparison between recent super-resolution (SR) methods, revealing which SR methods lie closest to the perception-distortion bound. 
\section{Distortion and perceptual quality}\label{sec:related}

Distortion and perceptual quality have been studied in many different contexts, and are sometimes referred to by different names. Let us briefly put past works in our context.

\subsection{Distortion (full-reference) measures}

Given a distorted image $\hat{x}$ and a ground-truth reference image $x$, full-reference distortion measures quantify the quality of $\hat{x}$ by its discrepancy to $x$. These measures are often called full reference image quality criteria because of the reasoning that if $\hat{x}$ is similar to $x$ and $x$ is of high quality, then $\hat{x}$ is also of high quality. However, as we show in this paper, this logic is not always correct. We thus prefer to call these measures distortion or dissimilarity criteria.

The most common distortion measure is the MSE, which is quite poorly correlated with semantic similarity between images \cite{wang2009mean}. Many alternative, more perceptual, distortion measures have been proposed over the years, including SSIM \cite{wang2004image}, MS-SSIM \cite{wang2003multiscale}, IFC \cite{sheikh2005information}, VIF \cite{sheikh2006image}, VSNR~\cite{chandler2007vsnr} and FSIM \cite{zhang2011fsim}. Recently, measures based on the $\ell_2$-distance between deep feature maps of a neural-net have been shown to capture more semantic similarities 
\cite{johnson2016perceptual,ledig2016photo,zhang2018unreasonable}.

\subsection{Perceptual quality} \label{sec:RelatedWorkPerceptualQuality}
The perceptual quality of an image $\hat{x}$ is the degree to which it looks like a natural image, and has nothing to do with its similarity to any reference image. In many image processing domains, perceptual quality has been associated with deviations from natural image statistics.

\subsubsection*{Human opinion based quality assessment}
Perceptual quality is commonly evaluated empirically by the mean opinion score of human subjects \cite{moorthy2011blind,mittal2012no}. Recently, it has become increasingly popular to perform such studies through real vs.~fake questionnaires \cite{isola2016image,zhang2016colorful,salimans2016improved,denton2015deep,dahl2017pixel,iizuka2016let,zhang2017real,guadarrama2017pixcolor}. These test the ability of a human observer to distinguish whether an image is real or the output of some algorithm. The probability of success $p_{\text{success}}$ of the optimal decision rule in this hypothesis testing task is known to be (see Appendix~\ref{ap:real-vs-fake} in the Supplementary Material)
\begin{equation}\label{eq:psuccess}
p_{\text{success}} = \tfrac{1}{2} d_{\text{TV}}(p_X,p_{\hat{X}}) + \tfrac{1}{2} ,
\end{equation}
where $d_{\text{TV}}(p_X,p_{\hat{X}})$ is the total-variation (TV) distance between the distribution $p_{\hat{X}}$ of images produced by the algorithm in question, and the distribution $p_X$ of natural images~\cite{nielsen2013hypothesis}. Note that $p_{\text{success}}$ decreases as the deviation between $p_{\hat{X}}$ and $p_X$ decreases, becoming $\tfrac{1}{2}$ (no better than a coin toss) when $p_{\hat{X}}=p_X$.

\subsubsection*{No-reference quality measures}
Perceptual quality can also be measured by an algorithm. In particular, no-reference measures quantify the perceptual quality $Q(\hat{x})$ of an image $\hat{x}$ \emph{without} depending on a reference image. These measures are commonly based on estimating deviations from natural image statistics. The works \cite{wang2005reduced,wang2006quality,li2009reduced} proposed perceptual quality indices based on the KL divergence between the distribution of the wavelet coefficients of $\hat{x}$ and that of natural scenes. This idea was further extended by the popular methods DIIVINE~\cite{moorthy2011blind}, BRISQUE~\cite{mittal2012no}, BLIINDS-II~\cite{saad2012blind} and NIQE~\cite{mittal2013making}, which quantify perceptual quality by various measures of deviation from natural image statistics in the spatial, wavelet and DCT domains.

\subsubsection*{GAN-based image restoration}
Most recently, GAN-based methods have demonstrated unprecedented perceptual quality in super-resolution \cite{ledig2016photo,sajjadi2017enhancenet,wang2018esrgan,shaham2019singan}, inpainting \cite{pathak2016context,yeh2017semantic,yu2018generative}, compression \cite{rippel2017real,agustsson2018generative,tschannen2018deep}, deblurring \cite{kupyn2018deblurgan} and image-to-image translation \cite{isola2016image,zhu2017unpaired,liu2017unsupervised}. This was accomplished by utilizing an adversarial loss, which minimizes some distance $d(p_X,p_{\hat{X}_{\text{GAN}}})$ between the distribution $p_{\hat{X}_{\text{GAN}}}$ of images produced by the generator and the distribution $p_X$ of images in the training dataset. A large variety of GAN schemes have been proposed, which minimize different distances between distributions. These include the Jensen-Shannon divergence \cite{goodfellow2014generative}, the Wasserstein distance \cite{arjovsky2017wasserstein}, and any $f$-divergence \cite{nowozin2016f}.

\subsubsection*{Single image quality vs. image ensemble quality}
A common measure of quality is the log-likelihood $Q_{\text{LL}}(\hat{x}) = \log (p_X(\hat{x}))$. However, this notion of quality evaluates each image individually, and therefore has shortcomings. As an example, consider a reconstruction algorithm that disregards the input image, and always outputs the same ``good-looking'' natural image that has a high likelihood. While this algorithm would rate very well by the average log-likelihood measure $\E[Q_{\text{LL}}(\hat{X})]$, it is obviously quite useless. An observer examining an \emph{ensemble} of outputs, would easily notice this flawed behavior. Therefore, in this paper we are more interested in ``ensemble quality''. The relation between the average log-likelihood and ``ensemble quality'' can be understood by noting that
\begin{equation}
\E_{\hat{X}\sim p_{\hat{X}}}[Q_{\text{LL}}(\hat{X})]=-d_{\text{KL}}(p_{\hat{X}},p_X) - H(p_{\hat{X}}).
\end{equation}
Here $d_{\text{KL}}$ is the Kullback-Leibler divergence and $H$ denotes entropy. The second term in this decomposition discourages diversity. 
Choosing to drop it, results in the distributional-divergence based quality measures described above.

\section{Problem formulation}\label{sec:perceptionDistortion}

In statistical terms, a natural image $x$ can be thought of as a realization from the distribution of natural images~$p_X$. In image restoration, we observe a degraded version $y$ relating to $x$ via some conditional distribution $p_{Y \vert X}$. In this paper we focus on non-invertible settings\footnote{\label{foot:invertible}By invertible we mean that the support of $p_{X|Y}(\cdot|y)$ is a singleton for almost all $y$'s (see Appendix \ref{ap:distortionProofNonunique} for a formal definition).}, where $x$ cannot be estimated from $y$ with zero error. This is typically the case in denoising, deblurring, inpaitning, super-resolution, etc. Given $y$, an image restoration algorithm produces an estimate $\hat{x}$ according to some distribution $p_{\hat{X}\vert Y}$. Note that this description is quite general in that it does not restrict the estimator $\hat{x}$ to be a deterministic function of $y$. This problem setting is illustrated in Fig.~\ref{fig:problemSetting}.

\begin{figure}
	\begin{center}
		\includegraphics[width=0.9\linewidth]{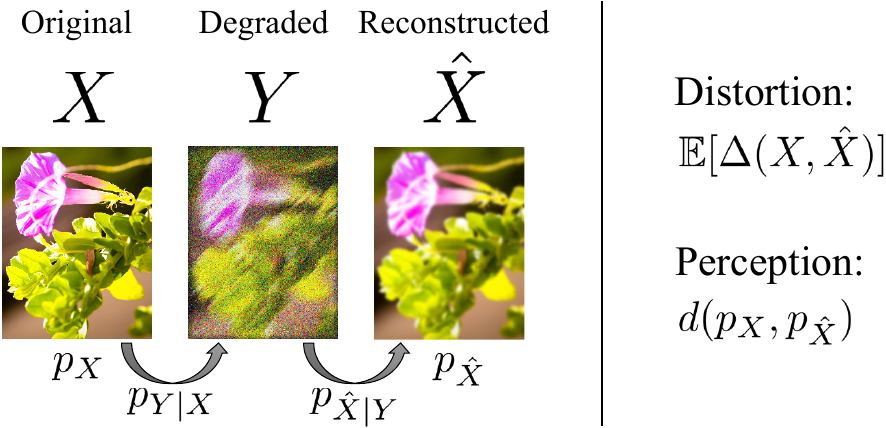}
	\end{center}
	\caption{\textbf{Problem setting.} Given an original image $x\sim p_X$, a degraded image $y$ is observed according to some conditional distribution $p_{Y \vert X}$. Given the degraded image $y$, an estimate $\hat{x}$ is constructed according to some conditional distribution $p_{\hat{X} \vert Y}$. Distortion is quantified by the mean of some distortion measure between $\hat{X}$ and $X$. The perceptual quality index corresponds to the deviation between $p_{\hat{X}}$ and $p_X$.}
	\label{fig:problemSetting}
\end{figure}

Given a full-reference dissimilarity criterion $\Delta(x,\hat{x})$, the average distortion of an estimator $\hat{X}$ is given by
\begin{equation}\label{eq:AverageDistortion}
\E[\Delta(X,\hat{X})],
\end{equation}
where the expectation is over the joint distribution $p_{X,\hat{X}}$. This definition aligns with the common practice of evaluating average performance over a database of degraded natural images. We assume that the dissimilarity criterion is such that $\Delta(x,\hat{x}) \ge 0$ with equality when $\hat{x} = x$. Note that some distortion measures, \eg SSIM, are actually \emph{similarity} measures (higher is better), yet can always be inverted (and shifted) to become dissimilarity measures.

As discussed in Sec.~\ref{sec:RelatedWorkPerceptualQuality}, the perceptual quality of an estimator $\hat{X}$ (as quantified \eg by real vs.~fake human opinion studies) is directly related to the distance between the distribution of its reconstructed images, $p_{\hat{X}}$, and the distribution of natural images, $p_X$. We thus define the perceptual quality index (lower is better) of an estimator $\hat{X}$ as
\begin{equation}\label{eq:PerceptualQualityIndex}
d(p_X,p_{\hat{X}}),
\end{equation}
where $d(p,q)$ is some divergence between distributions that satisfies $d(p,q) \ge 0$ with equality if $p=q$, \eg the KL divergence, TV distance, Wasserstein distance, etc. It should be pointed out that the divergence function $d(\cdot,\cdot)$ which best relates to human perception is a subject of ongoing research. Yet, our results below hold for (nearly) any divergence.

Notice that the best possible perceptual quality is obtained when the outputs of the algorithm follow the distribution of natural images (\ie $p_{\hat{X}}=p_X$). In this situation, by looking at the reconstructed images, it is impossible to tell that they were generated by an algorithm. However, not every estimator with this property is necessarily accurate. Indeed, we could achieve perfect perceptual quality by randomly drawing natural images that have nothing to do with the original ground-truth images. In this case the distortion would be quite large.

Our goal is to characterize the tradeoff between \eqref{eq:AverageDistortion} and~\eqref{eq:PerceptualQualityIndex}. But let us first see why minimizing the average distortion \eqref{eq:AverageDistortion}, does not necessarily lead to a low perceptual quality index \eqref{eq:PerceptualQualityIndex}. We start by illustrating this with the square-error distortion $\Delta(x,\hat{x})=\|x-\hat{x}\|^2$ and the $0-1$ distortion $\Delta(x,\hat{x})=1-\delta_{x,\hat{x}}$ (where $\delta$ is Kronecker's delta). We specifically illustrate that those measures are generally not distribution preserving in the following sense.
\begin{definition}
We say that a distortion measure $\Delta(\cdot,\cdot)$ is \emph{distribution preserving} at $p_{X,Y}$ if the estimator $\hat{X}$ minimizing the mean distortion \eqref{eq:AverageDistortion} satisfies $p_{\hat{X}}=p_X$.
\end{definition}
More details about those examples are provided in Appendix~\ref{ap:MMSE-MAP} (Supplementary Material). We then proceed to discuss this phenomenon for arbitrary distortions in Sec.~\ref{sec:arbitrary_dist}.

\subsection{The square-error distortion}\label{subsec:MMSEMAP}
The minimum mean square-error (MMSE) estimator is given by the posterior-mean $\hat{x}(y)=\E[X|Y=y]$. Consider the case $Y=X+N$, where $X$ is a discrete random variable with probability mass function
\begin{equation}\label{eq:XdiscreteExample}
	p_X(x) =
	\begin{cases}
	p_1 \quad \quad x= \pm 1,\\
	p_0 \quad \quad x=0,
	\end{cases}
\end{equation}
and $N \sim \mathcal{N}(0,1)$ is independent of $X$ (see Fig.~\ref{fig:exampleMAP}). In this setting, the MMSE estimate is given by
\begin{align}\label{eq:xMMSE}
\hat{x}_{\text{MMSE}}(y)&= \sum_{n\in\{-1,0,1\}}n \,\,p(X=n|y),
\end{align}
where
\begin{equation}\label{eq:PosteriorDiscrete}
p(X=n|y) = \frac{p_n \exp\{-\frac{1}{2}(y-n)^2\}}{\sum_{m\in\{-1,0,1\}}\limits p_m \exp\{-\frac{1}{2}(y-m)^2\}}.
\end{equation}
Notice that $\hat{x}_{\text{MMSE}}$ can take any value in the range $(-1,1)$, whereas $x$ can only take the discrete values $\{-1,0,1\}$. Thus, clearly, $p_{\hat{X}_{\text{MMSE}}}$ is very different from $p_X$, as illustrated in Fig.~\ref{fig:exampleMAP}. This demonstrates that minimizing the MSE distortion \emph{does not} generally lead to $p_{\hat{X}}\approx p_{X}$.

\begin{figure}
	\begin{center}
        \includegraphics[width=\linewidth]{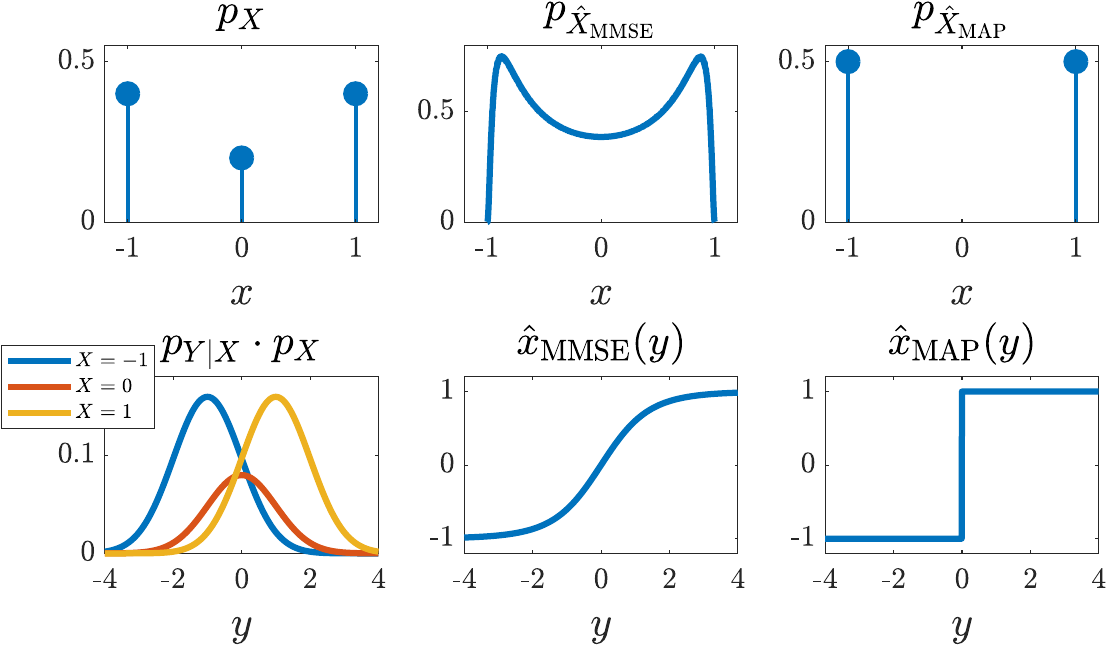}
	\end{center}
	\caption{\textbf{The distribution of the MMSE and MAP estimates.} In this example, $Y=X+N$, where $X \sim p_X$ and $N \sim \mathcal{N}(0,1)$. The distributions of both the MMSE and the MAP estimates deviate significantly from the distribution $p_X$.}
	\label{fig:exampleMAP}
\end{figure}

The same intuition holds for images. The MMSE estimate is an average over all possible explanations to the measured data, weighted by their likelihoods. However the average of valid images is not necessarily a valid image, so that the MMSE estimate frequently ``falls off'' the natural image manifold \cite{ledig2016photo}. This leads to unnatural blurry reconstructions, as illustrated in Fig.~\ref{fig:MMSE_MAP}. In this experiment, $x$ is a $280 \times 280$ image comprising $100$ smaller $28 \times 28$ digit images. Each digit is chosen uniformly at random from a dataset comprising $54$K images from the MNIST dataset \cite{MNIST} and an additional $5.4$K blank images. The degraded image $y$ is a noisy version of $x$. As can be seen, the MMSE estimator produces blurry reconstructions, which do not follow the statistics of the (binary) images in the dataset.

\begin{figure}
	\begin{center}
		\includegraphics[width=0.93\linewidth]{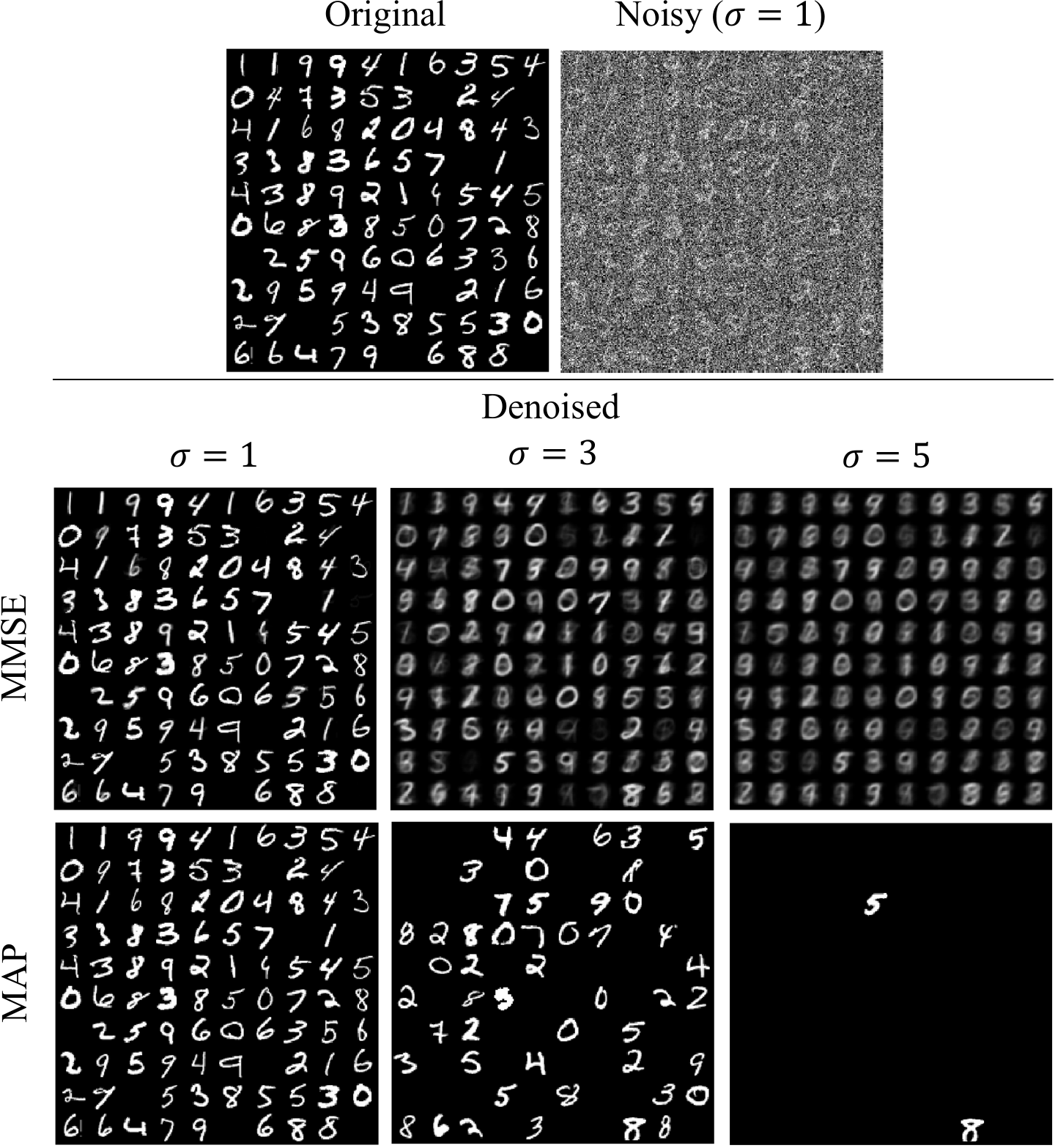}
	\end{center}
	\caption{\textbf{MMSE and MAP denoising.} Here, the original image consists of $100$ smaller images, chosen uniformly at random from the MNIST dataset enriched with blank images. After adding Gaussian noise ($\sigma=1,3,5$), the image is denoised using the MMSE and MAP estimators. In both cases, the estimates significantly deviate from the distribution of images in the dataset.}
	\label{fig:MMSE_MAP}
\end{figure}

\subsection{The $0-1$ distortion}\label{subsec:MMSEMAP2}
The discussion above may give the impression that unnatural estimates are mainly a problem of the square-error distortion, which causes averaging. One way to avoid averaging, is to minimize the binary $0-1$ loss, which restricts the estimator to choose $\hat{x}$ only from the set of values that $x$ can take. In fact, the minimum mean $0-1$ distortion is attained by the maximum-a-posteriori (MAP) rule, which is very popular in image restoration. However, as we exemplify next, the distribution of the MAP estimator also deviates from $p_X$. This behavior has also been studied in~\cite{nikolova2007model}.

Consider again the setting of \eqref{eq:XdiscreteExample}. In this case, the MAP estimate is given by
\begin{align}
\hat{x}_{\text{MAP}}(y)&= \argmax_{n\in\{-1,0,1\}}\limits p(X=n|y),
\end{align}
where $p(X=n|y)$ is as in \eqref{eq:PosteriorDiscrete}. Now, it can be easily verified that when $\log(p_1/p_0)>1/2$, we have $\hat{x}_{\text{MAP}}(y)=\text{sign}(y)$. Namely, the MAP estimator never predicts the value~$0$. Therefore, in this case, the distribution of the estimate is
\begin{equation}
	p_{\hat{X}_{\text{MAP}}}(\hat{x}) =
	\begin{cases}
	0.5 \quad \quad \hat{x}=+1,\\
	0.5 \quad \quad \hat{x}=-1,
	\end{cases}
\end{equation}
which is obviously different from $p_X$ of \eqref{eq:XdiscreteExample} (see Fig.~\ref{fig:exampleMAP}).

This effect can also be seen in the experiment of Fig.~\ref{fig:MMSE_MAP}. Here, the MAP estimates become increasingly dominated by blank images as the noise level rises, and thus clearly deviates from the underlying prior distribution.

\subsection{Arbitrary distortion measures}\label{sec:arbitrary_dist}
We saw that neither the square-error nor the $0-1$ loss are distribution preserving. That is, their minimization does not generally lead to $p_{\hat{X}}=p_X$ (\ie perfect perceptual quality). However these two examples do not yet preclude the existence of a distribution preserving distortion measure. Does there exist a measure whose minimization is guaranteed to lead to $p_{\hat{X}}=p_X$? If we limit ourselves to one single setting, then the answer may be positive. For example, in the setting of Fig.~\ref{fig:exampleMAP}, if $p_0$ of \eqref{eq:XdiscreteExample} equals~$0$, then the $0-1$ loss is distribution preserving as its minimization leads to an estimate satisfying $p_{\hat{X}}=p_X$. This illustrates that a distortion measure may be distribution preserving for certain underlying distributions $p_{X,Y}$ but not for others.

However, from a practical standpoint, we typically want our distortion measure to be adequate in more than one single setting. For example, if our goal is to train a neural network to perform denoising, then it is reasonable to expect that the same distortion measure be equally adequate as a loss function for different noise levels. In fact, we may also want to use the same distortion measure across different tasks (\eg super-resolution, deblurring, inpainting). The interesting question is, therefore, whether there exists a \emph{stably} distribution preserving distortion measure.
\begin{definition}
We say that a distortion measure $\Delta(\cdot,\cdot)$ is \emph{stably distribution preserving} at $p_{X,Y}$ if it is distribution preserving at all $\tilde{p}_{X,Y}$ in a TV $\varepsilon$-ball around $p_{X,Y}$ for some $\varepsilon>0$.
\end{definition}

As we show next, if the degradation is non-invertible, then no distortion metric can be stably distribution preserving (see proof in Appendix~\ref{ap:distortionProofNonunique}).
\begin{theorem}\label{thm:arbitraryDistortion}
	If $p_{X,Y}$ defines a non-invertible degradation, then $\Delta(\cdot,\cdot)$ is not a stably distribution preserving distortion at $p_{X,Y}$.
\end{theorem}

\section{The perception-distortion tradeoff}\label{sec:tradeOff}
We saw that for any distortion measure, a low distortion does not generally imply good perceptual-quality. An interesting question, then, is: What is the best perceptual quality that can be attained by an estimator with a prescribed distortion level?
\begin{definition}
The perception-distortion function of a signal restoration task is given by
\begin{equation}\label{eq:fAlpha}
P(D) = \min_{p_{\hat{X} \vert Y}} \, d(p_X,p_{\hat{X}}) \quad \text{s.t.} \quad  \E[\Delta(X,\hat{X})] \le D,
\end{equation}
where $\Delta(\cdot,\cdot)$ is a distortion measure and $d(\cdot,\cdot)$ is a divergence between distributions.
\end{definition}
In words, $P(D)$ is the minimal deviation between the distributions $p_X$ and $p_{\hat{X}}$ that can be attained by an estimator with distortion $D$.
To gain intuition into the typical behavior of this function, consider the following example.
\begin{example}\label{ex:scalarGaussian}
Suppose that $Y=X+N$, where $X \sim \mathcal{N}(0,1)$ and $N \sim \mathcal{N}(0,\sigma_N^2)$ are independent. Take $\Delta(\cdot,\cdot)$ to be the square-error distortion and $d(\cdot,\cdot)$ to be the KL divergence. For simplicity, let us focus on estimators of the form $\hat{X}=aY$. In this case, we can derive a closed form solution to Eq.~\eqref{eq:fAlpha} (see Appendix \ref{ap:scalarGaussian}), which is plotted for several noise levels $\sigma_N$ in Fig.~\ref{fig:scalarGaussianExample}. As can be seen, the minimal attainable $d_{\text{KL}}(p_X,p_{\hat{X}})$ drops as the maximal allowable distortion (MSE) increases. Furthermore, the tradeoff is convex and becomes more severe at higher noise levels $\sigma_N$.
\end{example}

\begin{figure}
	\begin{center}
        \includegraphics[width=0.9\linewidth]{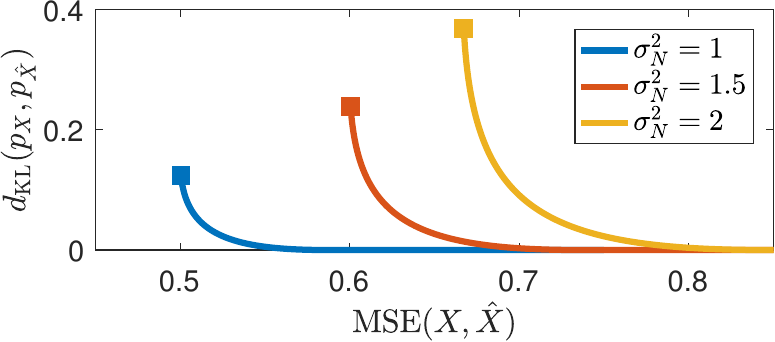}
	\end{center}
	\caption{\textbf{Plot of Eq.~\eqref{eq:fAlpha} for the setting of Example~\ref{ex:scalarGaussian}}. The minimal attainable KL distance between $p_X$ and $p_{\hat{X}}$ subject to a constraint on the maximal allowable MSE between $X$ and $\hat{X}$. Here, $Y=X+N$, where $X \sim \mathcal{N}(0,1)$ and $N \sim \mathcal{N}(0,\sigma_N)$, and the estimator is linear, $\hat{X}=aY$. Notice the clear trade-off: The perceptual index ($d_{\text{KL}}$) drops as the allowable distortion (MSE) increases. The graphs cut-off at the MMSE (marked by a square).}
	\label{fig:scalarGaussianExample}
\end{figure}

In general settings, it is impossible to solve~\eqref{eq:fAlpha} analytically. However, it turns out that the behavior seen in Fig.~\ref{fig:scalarGaussianExample} is typical, as we show next (see proof in Appendix \ref{ap:convexityProof}).

\begin{theorem}[The perception-distortion tradeoff]\label{lem:convexity}
Assume the problem setting of Section~\ref{sec:perceptionDistortion}. If $d(p,q)$ of \eqref{eq:PerceptualQualityIndex} is convex in its second argument\footnote{That is, $d(p,\lambda q_1 + (1-\lambda) q_2) \le \lambda d(p,q_1) + (1-\lambda) d(p,q_2)$ for any three distributions $p,q_1,q_2$ and any $\lambda \in [0,1]$.},
then the perception-distortion function $P(D)$ of \eqref{eq:fAlpha} is
\vspace{-0.15cm}
\begin{enumerate}
	\item monotonically non-increasing;
	\item convex.
\end{enumerate}
\end{theorem}

Note that Theorem~\ref{lem:convexity} requires no assumptions on the distortion measure $\Delta(\cdot,\cdot)$. This implies that a tradeoff between perceptual quality and distortion exists for \emph{any distortion measure}, including \eg MSE, SSIM, square error between VGG features \cite{johnson2016perceptual,ledig2016photo}, etc. Yet, this does not imply that all distortion measures have the same perception-distortion function. Indeed, as we demonstrate in Sec.~\ref{sec:practicalMethod}, the tradeoff tends to be less severe for distortion measures that capture semantic similarities between images.

The convexity of $P(D)$ implies that the tradeoff is more severe at the low-distortion and at the high-perceptual-quality extremes. This is particularly important when considering the TV divergence which is associated with the ability to distinguish between real vs.~fake images (see Sec.~\ref{sec:RelatedWorkPerceptualQuality}). Since $P(D)$ is steeper at the low-distortion regime, any \emph{small} improvement in distortion for an algorithm whose distortion is already low, must be accompanied by a \emph{large} degradation in the ability to fool a discriminator. Similarly, any \emph{small} improvement in the perceptual quality of an algorithm whose perceptual index is already low, must be accompanied by a \emph{large} increase in distortion. Let us comment that the assumption that $d(p,q)$ is convex, is not very limiting. For instance, any $f$-divergence (\eg KL, TV, Hellinger, $\mathcal{X}^2$) as well as the Renyi divergence, satisfy this assumption \cite{csiszar2004information,van2014renyi}. In any case, the function $P(D)$ is monotonically non-increasing even without this assumption.

\subsection{Bounding the Perception-Distortion function}\label{sec:bounding}

Several past works attempted to answer the question: What is the minimal attainable distortion $D_{\min}$ in various restoration tasks? \cite{levin2011natural,levin2012patch,chatterjee2010denoising,chatterjee2011practical,baker2002limits}. This corresponds to the value 
\begin{equation}\label{eq:Dmin}
D_{\min} =  \min_{p_{\hat{X}|Y}} \E[\Delta(X,\hat{X})],
\end{equation}
which is the horizontal coordinate of the leftmost point on the perception-distortion function. However, as the minimum distortion estimator is generally not distribution preserving (Sec.~\ref{sec:arbitrary_dist}), an important complementary question is: What is the minimal distortion that can be attained by an estimator \emph{having perfect perceptual quality}? This corresponds to the value
\begin{equation}\label{eq:Dmax}
D_{\max} =  \min_{p_{\hat{X}|Y}} \E[\Delta(X,\hat{X})] \quad \text{s.t.} \quad p_{\hat{X}} = p_X,
\end{equation}
which is the horizontal coordinate of the point where the perception-distortion function first touches the horizontal axis (see Fig.~\ref{fig:minMaxDist}).

Observe that perfect perceptual quality ($p_{\hat{X}} = p_X$) is always attainable, for example by drawing $\hat{x}$ from $p_X$ independently of the input $y$. 
This method, however, ignores the input and is thus not good in terms of distortion. 
It turns out that perfect perceptual quality can generally be achieved with a significantly lower MSE distortion, as we show next  (see proof in Appendix \ref{ap:boundProof}).

\begin{figure}
	\begin{center}
		\includegraphics[width=0.85\linewidth]{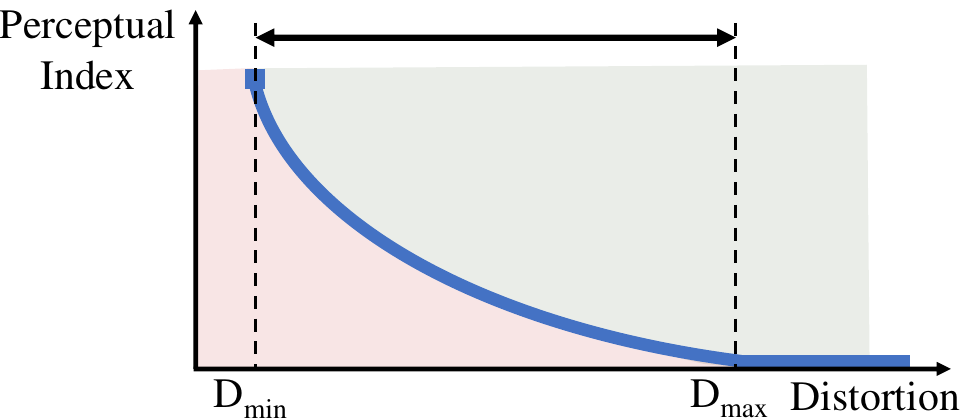}
	\end{center}
	\caption{\textbf{Bounding the perception-distortion function.} The distance between $D_{\min}$ and $D_{\max}$ is the increase in distortion which is needed to obtain perfect perceptual quality. For the MSE, Theorem \ref{thm:bound} proves this will never be more than a factor of $2$ (which is $3$dB in terms of PSNR).}
	\label{fig:minMaxDist}
\end{figure}

\begin{theorem}\label{thm:bound}
	For the square error distortion $\Delta(x,\hat{x}) = \|\hat{x}-x\|^2$,
	\begin{equation}
	D_{\max}\leq 2D_{\min},
	\end{equation}	
	where $D_{\min}$ and $D_{\max}$ are defined by \eqref{eq:Dmin} and \eqref{eq:Dmax}, respectively. This bound is attained by the estimator $\hat{X}$ defined through
\begin{equation}\label{eq:xpost}
p_{\hat{X}|Y}(x|y) = p_{X|Y}(x|y),
\end{equation}
which achieves $p_{\hat{X}}=p_X$ and has an MSE of $2D_{\min}$.
\end{theorem}

In simple words, Theorem \ref{thm:bound} states that one would never need to sacrifice more than $3$dB in PSNR to obtain perfect perceptual quality. This can be achieved by drawing $\hat{x}$ from the posterior distribution $p_{X|Y}$. Interestingly, such a degradation was indeed incurred by all super-resolution methods that achieved state-of-the-art perceptual quality to date. This can be seen in Fig.~\ref{fig:noRefMethods1}, where the RMSE of the algorithms with the lowest perceptual index is nearly a factor of $\sqrt{2}$ larger than the RMSE of the methods with the lowest RMSE (see also \cite{ledig2016photo,sajjadi2017enhancenet}). However, note that this bound is generally not tight. For example, in the scalar Gaussian toy example of Fig.~\ref{fig:scalarGaussianExample}, $D_{\max}$ can be quite smaller than $2D_{\min}$, depending on the noise level.

\subsection{Connection to rate-distortion theory}\label{sec:rateDistortion}
The perception-distortion tradeoff is closely related to the well-established rate-distortion theory \cite{cover2012elements}. This theory characterizes the tradeoff between the bit-rate required to communicate a signal, and the distortion incurred in the signal's reconstruction at the receiver. More formally, the rate-distortion function of a signal $X$ is defined by
\begin{equation}\label{eq:rateDistortion}
R(D) = \min_{p_{\hat{X} \vert X}} \, I(X;\hat{X}) \quad \text{s.t.} \quad  \E[\Delta(X,\hat{X})] \le D,
\end{equation}
where $I(\!X;\hat{X}\!)$ is the mutual information between $X$ and $\hat{X}$.

There are, however, several key differences between the two tradeoffs. First, in rate-distortion the optimization is over all conditional distributions $p_{\hat{X} \vert X}$, \ie given the \emph{original} signal. In the perception-distortion case, the estimator has access only to the degraded signal $Y$, so that the optimization is over the conditional distributions $p_{\hat{X} \vert Y}$, which is more restrictive. In other words, the perception-distortion tradeoff depends on the degradation $p_{Y|X}$, and not only on the signal's distribution $p_X$ (see Example~\ref{ex:scalarGaussian}). Second, in rate-distortion the rate is quantified by the mutual information $I(X;\hat{X})$, which depends on the joint distribution $p_{X,\hat{X}}$. In our case, perception is quantified by the similarity between $p_X$ and $p_{\hat{X}}$, which does not depend on their joint distribution. Lastly, mutual information is inherently convex, while the convexity of the perception-distortion curve is guaranteed only when $d(\cdot,\cdot)$ is convex.

While the two tradeoffs are different, it is important to note that perceptual quality does play a role in lossy compression, as evident from the success of recent GAN based compression schemes \cite{tschannen2018deep, agustsson2018generative, santurkar2018generative}. Theoretically, its effect can be studied through the rate-distortion-perception function \cite{blau2019rethinking,matsumoto2018introducing,matsumoto2018rate}, which is an extension of the rate-distortion function \eqref{eq:rateDistortion} and the perception-distortion function \eqref{eq:fAlpha}, characterizing the triple tradeoff between rate, distortion, and perceptual quality. 

\section{Traversing the tradeoff with a GAN}\label{sec:WGAN}
There exists a systematic way to design estimators that approach the perception-distortion curve: Using GANs. Specifically, motivated by \cite{ledig2016photo,pathak2016context,yeh2017semantic,sajjadi2017enhancenet,rippel2017real,isola2016image}, restoration problems can be approached by modifying the loss of the generator of a GAN to be
\begin{equation}\label{eq:GANloss}
\ell_\text{gen} = \ell_{\text{distortion}} + \lambda \, \ell_{\text{adv}},
\end{equation}
where $\ell_{\text{distortion}}$ is the distortion between the original and reconstructed images, and $\ell_{\text{adv}}$ is the standard GAN adversarial loss. It is well known that $\ell_{\text{adv}}$ is proportional to some divergence $d(p_X,p_{\hat{X}})$ between the generator and data distributions \cite{goodfellow2014generative,arjovsky2017wasserstein,nowozin2016f} (the type of divergence depends on the loss). Thus,~\eqref{eq:GANloss} in fact approximates the objective
\begin{equation}
\ell_\text{gen} \approx \E[\Delta(x,\hat{x})] + \lambda \, d(p_X,p_{\hat{X}}).
\end{equation}
Viewing $\lambda$ as a Lagrange multiplier, it is clear that minimizing $\ell_\text{gen}$ is equivalent to minimizing \eqref{eq:fAlpha} for some $D$. Varying $\lambda$ corresponds to varying $D$, thus producing estimators along the perception-distortion function.

Let us use this approach to explore the perception-distortion tradeoff for the digit denoising example of Fig.~\ref{fig:MMSE_MAP} with $\sigma=3$. We train a Wasserstein GAN (WGAN) based denoiser \cite{arjovsky2017wasserstein,gulrajani2017improved} with an MSE distortion loss $\ell_{\text{distortion}}$. Here, $\ell_{\text{adv}}$ is proportional to the Wasserstein distance $d_W(p_X,p_{\hat{X}})$ between the generator and data distributions. The WGAN has the valuable property that its discriminator (critic) loss is an accurate estimate (up to a constant factor) of $d_W(p_X,p_{\hat{X}})$ \cite{arjovsky2017wasserstein}. This allows us to easily compute the perceptual quality index of the trained denoiser. We obtain a set of estimators with several values of $\lambda\in[0,0.3]$. For each denoiser, we evaluate the perceptual quality by the final discriminator loss. As seen in Fig.~\ref{fig:WGAN}, the curve connecting the estimators on the perception-distortion plane is monotonically decreasing. Moreover, it is associated with estimates that gradually transition from blurry and accurate to sharp and inaccurate. This curve obviously does not coincide with the analytic bound \eqref{eq:fAlpha} (illustrated by a dashed line). However, it seems to be adjacent to it. This is indicated by the fact that the left-most point of the WGAN curve is very close to the left-most point of the theoretical bound, which corresponds to the MMSE estimator. See Appendix \ref{ap:WGANdetails} for the WGAN training details and architecture.

\begin{figure}
	\begin{center}
		\includegraphics[width=\linewidth]{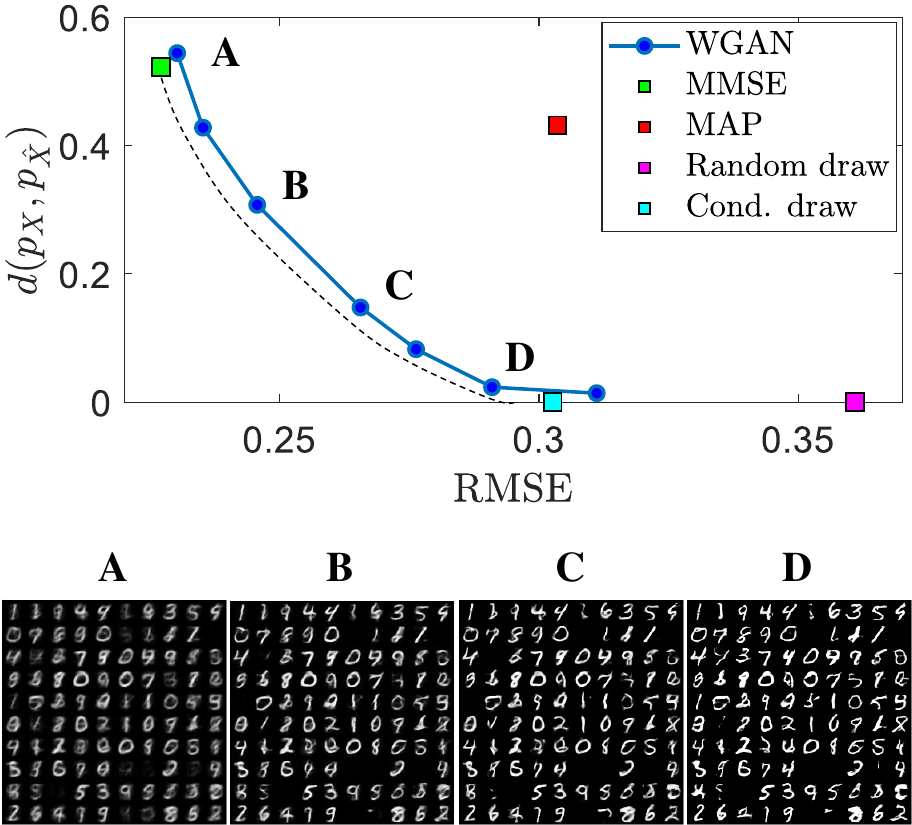}
	\end{center}
	\caption{\textbf{Image denoising utilizing a GAN.} A Wasserstein GAN was trained to denoise the images of the experiment in Fig.~\ref{fig:MMSE_MAP}. The generator loss $l_\text{gen} = l_{\text{MSE}} + \lambda \, l_{\text{adv}}$ consists of a perceptual quality (adversarial) loss and a distortion (MSE) loss, where $\lambda$ controls the trade-off between the two. For each $\lambda \in [0,0.3]$, the graph depicts the distortion (MSE) and perceptual quality (Wasserstein distance between $p_X$ and $p_{\hat{X}}$). The curve connecting the estimators is a good approximation to the theoretical perception-distortion tradeoff (illustrated by a dashed line).}
	\label{fig:WGAN}
\end{figure}

Besides the MMSE estimator, Figure \ref{fig:WGAN} also includes the MAP estimator, the random draw estimator $\hat{x}\sim p_X$ (which ignores the noisy image $y$), and the conditional draw estimator of \eqref{eq:xpost}. The perceptual quality of these estimators is evaluated, as above, by the final loss of the WGAN discriminator \cite{arjovsky2017wasserstein}, trained (without a generator) to distinguish between the estimators' outputs and images from the dataset.
Note that the denoising WGAN estimator (D) achieves the same distortion as the MAP estimator, but with far better perceptual quality. Furthermore, it achieves nearly the same perceptual quality as the random draw estimator, but with a significantly lower distortion.

\section{Practical method for evaluating algorithms}\label{sec:practicalMethod}

Certain applications may require low-distortion (\eg in medical imaging), while others may prefer superior perceptual quality. How should image restoration algorithms be evaluated, then?
\begin{definition}
We say that Algorithm A \emph{dominates} Algorithm B if it has better perceptual quality \emph{and} less distortion.
\end{definition}
Note that if Algorithm A is better than B in only one of the two criteria, then neither $A$ dominates $B$ nor $B$ dominates $A$. Therefore, among a group of algorithms, there may be a large subset which can be considered equally good.
\begin{definition}
We say that an algorithm is \emph{admissible} among a group of algorithms, if it is not dominated by any other algorithm in the group.
\end{definition}
As shown in Figure~\ref{fig:dominantAdmissable}, these definitions have very simple interpretations when plotting algorithms on the perception-distortion plane. In particular, the admissible algorithms in the group, are those which lie closest to the perception-distortion bound.

\begin{figure}
	\begin{center}
		\includegraphics[width=0.85\linewidth]{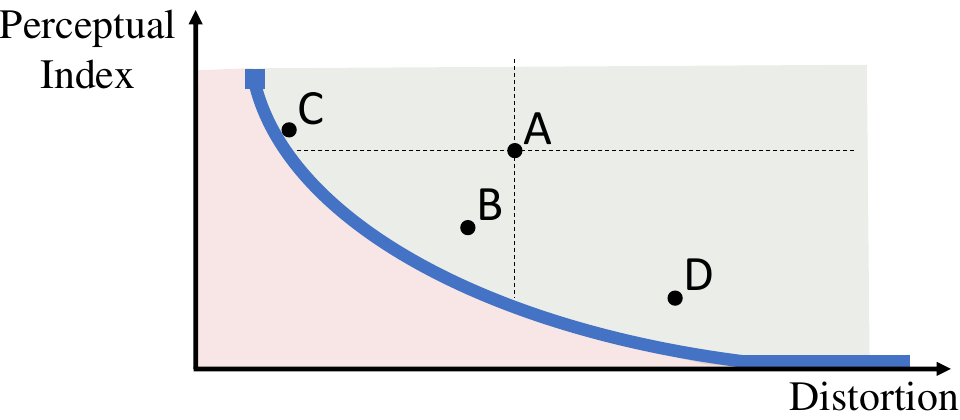}
	\end{center}
	\caption{\textbf{Dominance and admissibility.} Algorithm A is dominated by Algorithm B, and is thus inadmissible. Algorithms B, C and D are all admissible, as they are not dominated by any algorithm.}
	\label{fig:dominantAdmissable}
\end{figure}

\begin{figure*}
	\begin{center}
		\includegraphics[width=\linewidth]{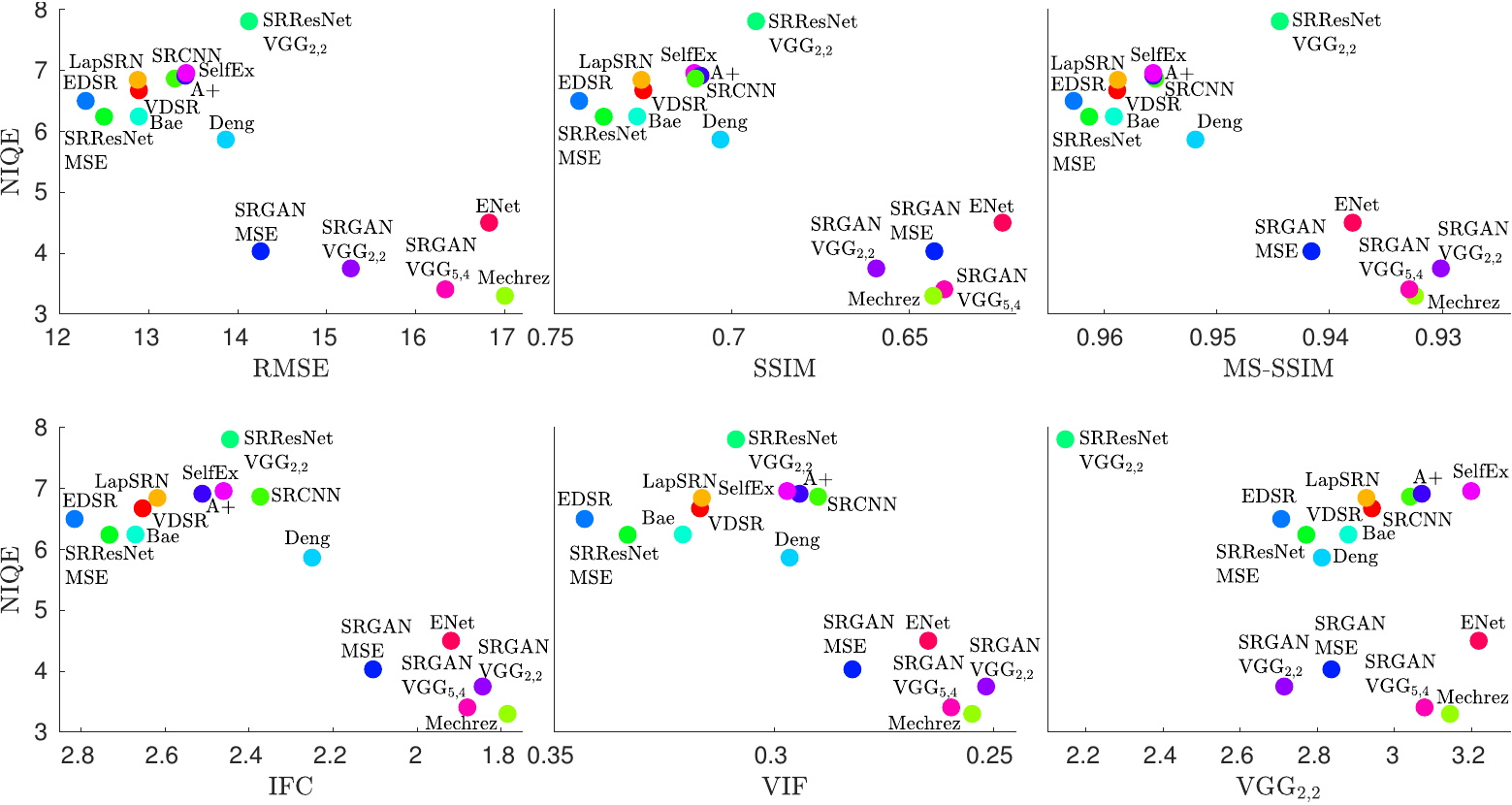}
	\end{center}
	\caption{\textbf{Perception-distortion evaluation of SR algorithms.} We plot $16$ algorithms on the perception-distortion plane. Perception is measured by the NR metric NIQE \cite{mittal2013making}. Distortion is measured by the common full-reference metrics RMSE, SSIM, MS-SSIM, IFC, VIF and VGG$_{2,2}$. In all plots, the lower left corner is blank, revealing an unattainable region in the perception-distortion plane. In proximity of the unattainable region, an improvement in perceptual quality comes at the expense of higher distortion.}
	\label{fig:noRefMethods1}
\end{figure*}

As discussed in Sec.~\ref{sec:related}, distortion is measured by \emph{full}-reference (FR) metrics, \eg \cite{wang2004image,wang2003multiscale,sheikh2005information,sheikh2006image,chandler2007vsnr,zhang2011fsim,johnson2016perceptual}. The choice of the FR metric, depends on the type of similarities we want to measure (per-pixel, semantic, etc.). Perceptual quality, on the other hand, is ideally quantified by collecting human opinion scores, which is time consuming and costly \cite{moorthy2011blind,saad2012blind}. Instead, the divergence $d(p_X,p_{\hat{X}})$ can be computed, for instance by training a discriminator net (see Sec.~\ref{sec:WGAN}). However, this requires \emph{many} training images and is thus also time consuming. A practical alternative is to utilize \emph{no}-reference (NR) metrics, \eg  \cite{mittal2012no,mittal2013making,saad2012blind,moorthy2011blind,ye2012unsupervised,kang2014convolutional,ma2017learning}, which quantify the perceptual quality of an image \emph{without} a corresponding original image. In scenarios where NR metrics are highly correlated with human mean-opinion-scores (\eg $4\times$ super-resolution \cite{ma2017learning}), they can be used as a fast and simple method for approximating the perceptual quality of an algorithm\footnote{In scenarios where NR metrics are inaccurate (\eg blind deblurring with large blurs \cite{lai2016comparative,liu2013no}), the perceptual metric should be human-opinion-scores or the loss of a discriminator trained to distinguish the algorithms' outputs from natural images.}.

\begin{figure*}
	\begin{center}
		\includegraphics[width=\linewidth,trim={0 9cm 0 8.4cm},clip]{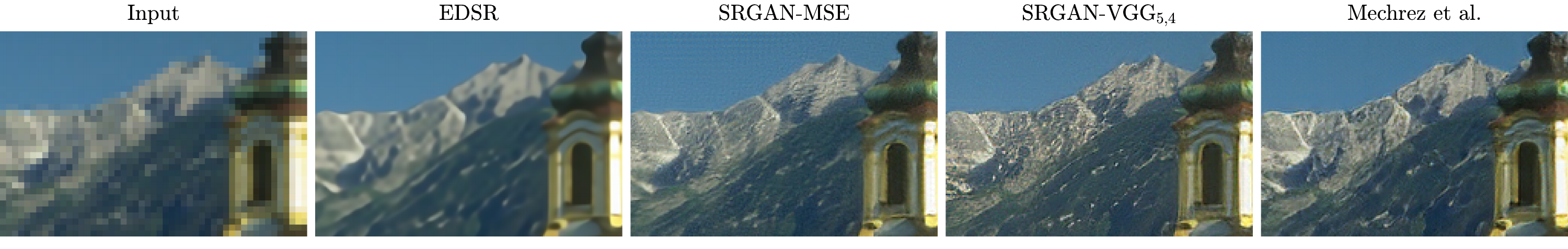}
	\end{center}
	\caption{\textbf{Visual comparison of algorithms closest to the perception-distortion bound.} The algorithms are ordered from low to high distortion (as evaluated by RMSE, MS-SSIM, IFC, VIF). Notice the co-occurring increase in perceptual quality.}
	\label{fig:comparison1}
\end{figure*}

We use this approach to evaluate $16$ SR algorithms in a $4\times$ magnification task, by plotting them on the perception-distortion plane (Fig.~\ref{fig:noRefMethods1}). We measure perceptual quality using the NR metric NIQE \cite{mittal2013making}, which was shown to correlate well with human opinion scores in a recent SR challenge \cite{blau2018pirm} (see Appendix \ref{sec:SR_details} for experiments with the NR metrics  BRISQUE \cite{mittal2012no}, BLIINDS-II \cite{saad2012blind} and the recent NR metric by Ma \etal \cite{ma2017learning}). We measure distortion by the five common FR metrics RMSE, SSIM \cite{wang2004image}, MS-SSIM \cite{wang2003multiscale}, IFC \cite{sheikh2005information} and VIF \cite{sheikh2006image}, and additionally by the recent $\text{VGG}_{2,2}$ metric (the distance in the feature space of a VGG net) \cite{ledig2016photo,johnson2016perceptual}. To conform to previous evaluations, we compute all metrics on the y-channel after discarding a 4-pixel border (except for VGG$_{2,2}$, which is computed on RGB images). Comparisons on color images can be found in Appendix \ref{sec:SR_details}. The algorithms are evaluated on the BSD100 dataset \cite{martin2001database}. The evaluated algorithms include: A+ \cite{timofte2014a+}, SRCNN \cite{dong2014learning}, SelfEx \cite{huang2015single}, VDSR~\cite{kim2016accurate}, Johnson \etal~\cite{johnson2016perceptual}, LapSRN \cite{lai2017deep}, Bae \etal~\cite{bae2017beyond} (``primary'' variant), EDSR \cite{lim2017enhanced}, SRResNet variants which optimize MSE and $\text{VGG}_{2,2}$ \cite{ledig2016photo}, SRGAN variants which optimize MSE, $\text{VGG}_{2,2}$, and $\text{VGG}_{5,4}$, in addition to an adversarial loss \cite{ledig2016photo}, ENet \cite{sajjadi2017enhancenet} (``PAT'' variant), Deng \cite{deng2018enhancing} ($\gamma=0.55$), and Mechrez \etal~\cite{Mechrez2018SR}.

Interestingly, the same pattern is observed in all plots:
(i) The lower left corner is blank, revealing an unattainable region in the perception-distortion plane. (ii) In proximity of this blank region, NR and FR metrics are \emph{anti-correlated}, indicating a tradeoff between perception and distortion. Notice that the tradeoff exists even for the IFC, VIF and VGG$_{2,2}$ measures, which are considered to capture visual quality better than MSE and SSIM.

Figure~\ref{fig:comparison1} depicts the outputs of several algorithms lying closest to the perception-distortion bound in the IFC graph in Fig. \ref{fig:noRefMethods1}. While the images are ordered from low to high distortion (according to IFC), their perceptual quality clearly improves from left to right.

Both FR and NR measures are commonly validated by calculating their correlation with human opinion scores, based on the assumption that both should be correlated with perceptual quality. However, as Fig.~\ref{fig:Ma_IFC_correlation} shows, while FR measures can be well-correlated with perceptual quality when distant from the unattainable region, this is clearly not the case when approaching the perception-distortion bound. In particular, all tested FR methods are inconsistent with human opinion scores which found the SRGAN to be superb in terms of perceptual quality \cite{ledig2016photo}, while NR methods successfully determine this. We conclude that image restoration algorithms should always be evaluated by a pair of NR and FR metrics, constituting a reliable, reproducible and simple method for comparison, which accounts for both perceptual quality and distortion. This evaluation method was demonstrated and validated by a human opinion study in the 2018 PIRM super-resolution challenge \cite{blau2018pirm}.

\begin{figure}
	\begin{center}
		\includegraphics[width=0.95\linewidth]{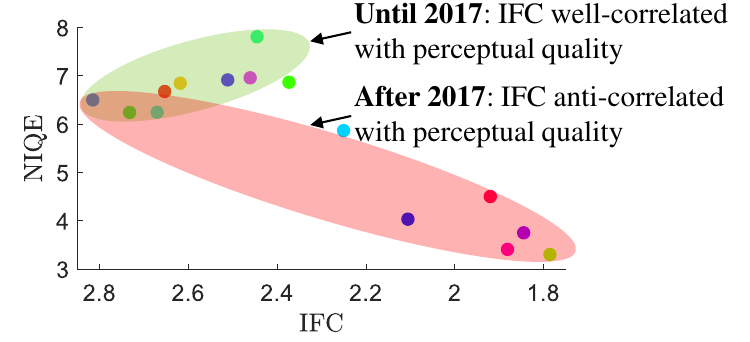}
	\end{center}
	\caption{\textbf{Correlation between distortion and perceptual quality.} In proximity of the perception-distortion bound, distortion and perceptual quality are \emph{anti-correlated}. However, correlation is possible at distance from the bound.}
	\label{fig:Ma_IFC_correlation}
\end{figure}

Up until 2016, SR algorithms occupied only the upper-left section of the perception-distortion plane. Nowadays, emerging techniques are exploring new regions in this plane. The SRGAN, ENet, Deng, Johnson \etal~and Mechrez \etal~methods are the first (to our knowledge) to populate the high perceptual quality region. In the near future we will most likely witness continued efforts to approach the perception-distortion bound, not only in the low-distortion region, but throughout the entire plane.

\section{Conclusion}
We proved and demonstrated the counter-intuitive phenomenon that distortion and perceptual quality are at odds with each other. Namely, the lower the distortion of an algorithm, the more its distribution must deviate from the statistics of natural scenes. We showed empirically that this tradeoff exists for many popular distortion measures, including those considered to be well-correlated with human perception. Therefore, any distortion measure alone, is unsuitable for assessing image restoration methods. Our novel methodology utilizes a pair of NR and FR metrics to place each algorithm on the perception-distortion plane, facilitating a more informative comparison of image restoration methods.

\vspace{0.2cm}
\noindent \textbf{Acknowledgements\ \ \ }
This research was supported by the Israel Science Foundation (grant no.~852/17), and by the Technion Ollendorff Minerva Center.

\bibliographystyle{IEEEtran}

\appendices

\section{Real-vs.-fake user studies and hypothesis testing}\label{ap:real-vs-fake}
We assume the setting where an observer is shown a real image (a draw from $p_X$) or an algorithm output (a draw from $p_{\hat{X}}$), with a prior probability of  $0.5$ each. The task is to identify which distribution the image was drawn from ($p_X$ or $p_{\hat{X}}$) with maximal probability of success. This is the setting of the Bayesian hypothesis testing problem, for which the maximum a-posteriori (MAP) decision rule minimizes the probability of error (see Section 1 in \cite{nielsen2013hypothesis}). When there are two possible hypotheses with equal probabilities (as in our setting), the relation between the probability of error and the total-variation distance between $p_X$ and $p_{\hat{X}}$ in~\eqref{eq:psuccess} can be easily derived (see Section 2 in \cite{nielsen2013hypothesis}).

\section{The MMSE and MAP examples of Sec.~\ref{sec:perceptionDistortion}}\label{ap:MMSE-MAP}
Sections \ref{subsec:MMSEMAP} and \ref{subsec:MMSEMAP2} exemplify that the MSE and the $0-1$ loss are not distribution preserving in the setting of estimating a discrete random variable (vector) $X$ from its noisy version $Y=X+N$, where $N\sim \mathcal{N}(0,\sigma^2 I)$ is independent of $X$. Since the conditional distribution of $Y$ given $X=x$ is $\mathcal{N}(x,\sigma^2 I)$, the MMSE estimator is given by
\begin{align}
\hat{x}_{\text{MMSE}} (y) &= \E[X \vert Y=y] \nonumber \\
&= \sum_x x p(x \vert y)  \nonumber \\
&= \sum_x x \frac{  p(y \vert x)p(x)}{\sum_{x'} p(y \vert x')p(x')} \nonumber \\
&= \sum_x x \frac{\exp(-\frac{1}{2\sigma^2}\|y-x\|^2)p(x)}{\sum_{x'} \exp(-\frac{1}{2\sigma^2}\|y-x'\|^2)p(x')},
\end{align}
and the MAP estimator is given by
\begin{align}
\hat{x}_{\text{MAP}}(y) &= \argmax_x p(x \vert y) \nonumber \\
&= \argmin_x -\log (p(y \vert x) p(x)) \nonumber \\
&= \argmin_x \frac{1}{2\sigma^2}\|y-x\|^2 - \log (p(x)).
\end{align}

In the example of Fig.~\ref{fig:MMSE_MAP}, $x$ is a $280 \times 280$ binary image comprising $28\times 28$ blocks chosen uniformly at random from a finite database. Since the noise $N$ is i.i.d., each $28\times 28$ block of $y$ can be denoised separately, both in the case of the MSE criterion and in the case of MAP. For each block, we have $p(x) = 1 / 59400$ for the non-blank images and $p(x) = 1 / 11$ for the blank image.

In the trinary example \eqref{eq:XdiscreteExample}, we calculate the distribution of the MMSE estimate (Fig.~\ref{fig:exampleMAP}) by
\begin{equation}
p_{\hat{X}_{\text{MMSE}}}(\hat{x}) = p_Y(\hat{x}_{\text{MMSE}}^{-1}(\hat{x})) \left\vert \frac{d}{d\hat{x}} \hat{x}_{\text{MMSE}}^{-1}(\hat{x}) \right\vert
\end{equation}
where the inverse of $\hat{x}_{\text{MMSE}}(y)$ (see \eqref{eq:xMMSE}) and its derivative are calculated numerically, and $p_Y(y) = \sum_x p(y \vert x) p(x)$ with $p(y \vert x) \sim \mathcal{N}(x,1)$ and $p(x)$ of \eqref{eq:XdiscreteExample}.

\section{Proof of Theorem \ref{thm:arbitraryDistortion}}\label{ap:distortionProofNonunique}
We will show that a stably distribution preserving optimal estimator is necessarily unique. At the same time, we will show that a non-invertible degradation implies that this optimal estimator is non-unique. Specifically, we use the following definitions.

\begin{definition}\label{def:nonInvertibleDeg}
	We say that a degradation is \emph{not invertible} if $p_{X|Y}(x|y)>0$ for all $(x,y)\in \mathcal{S}_x\times \mathcal{S}_y$, where $\mathcal{S}_x$ is a non-singleton set and $\mathcal{S}_y$ satisfies $\mathbb{P}(Y\in\mathcal{S}_y)>0$.
\end{definition}

\begin{definition}\label{def:notUniqueEstimator}
	We say that the optimal estimator is \emph{not unique} if there exist two estimators, $p_{\hat{X}_1|Y}$ and $p_{\hat{X}_2|Y}$ that minimize the mean distortion~\eqref{eq:AverageDistortion} and differ from one another in the sense that
	\begin{equation}
	    d_{\text{TV}}\left(p_{\hat{X}_1|Y}(\cdot|y),p_{\hat{X}_2|Y}(\cdot|y)\right)>0 \quad \forall y \in \mathcal{S}_y
	\end{equation}
	where $\mathcal{S}_y$ is a set that satisfies $\mathbb{P}(Y\in\mathcal{S}_y)>0$.
\end{definition}

The outline of the proof of Theorem~\ref{thm:arbitraryDistortion} will be as follows:
\begin{enumerate}
    \item In Lemma~\ref{lem:estimatorIsPosterior} we will show that if the distortion measure is stably distribution preserving, then the optimal estimator $\hat{X}^*$ is uniquely defined by $p_{\hat{X}^*|Y} = p_{X|Y}$.
    \item In Lemma~\ref{lem:estimatorNonUnique} we will show that if the estimator $\hat{X}^*$ defined by $p_{\hat{X}^*|Y} = p_{X|Y}$ is an optimal estimator and the degradation is non-invertible, then the optimal estimator is non-unique.
    \item This leads to a contradiction, proving that there does not exist a stably distribution preserving distortion metric if the degradation in non-invertible.
\end{enumerate}

\begin{lemma}\label{lem:estimatorIsPosterior}
	If the distortion measure $\Delta(\cdot,\cdot)$ is stably distribution preserving at $p_{X,Y}$, then the optimal estimator $\hat{X}^*$ that minimizes the mean distortion~\eqref{eq:AverageDistortion} is uniquely defined by $p_{\hat{X}^*|Y} = p_{X|Y}$.
\end{lemma}

\begin{proof}
We start by noting that the optimal estimator $p_{\hat{X}|Y}$ depends only on $p_{X|Y}$ and not on $p_Y$. Indeed, since $X$ and $\hat{X}$ are independent given $Y$, the mean distortion can be written as
\begin{align}\label{eq:meanDistortionExplicit}
\E[\Delta(X,\hat{X})]\! &=\!\iiint \!\!\Delta(x,\hat{x}) p_{X|Y}\!(x|y)p_{\hat{X}|Y}\!(\hat{x}|y)p_Y\!(y) dx d\hat{x} dy \nonumber\\
&= \int \left(\int f(\hat{x},y) p_{\hat{X}|Y}(\hat{x}|y) d\hat{x} \right) p_Y(y) dy,
\end{align}
where we defined
\begin{equation}\label{eq:definef}
f(\hat{x},y) = \int \Delta(x,\hat{x}) p_{X|Y}(x|y)dx.
\end{equation}
Therefore, the optimal $p_{\hat{X}|Y}$ is that which minimizes $\int f(\hat{x},y) p_{\hat{X}|Y}(\hat{x}|y) d\hat{x}$ for each $y$. Since $f(\hat{x},y)$ depends only on $p_{X|Y}$,
the optimal estimator depends only on $p_{X|Y}$.

Next, we observe that if a distortion measure is stably distribution preserving at $p_{X,Y}$, then there exists an $\alpha\in(0,1)$ such that the measure is distribution preserving at any perturbed joint distribution of the form $\tilde{p}_{X,Y} = p_{X|Y} \tilde{p}_Y $, where
\begin{equation}\label{eq:perturbedPy}
\tilde{p}_Y = \alpha p_Y + (1-\alpha) q
\end{equation}
and $q$ is any distribution. 
That is, we take a perturbed joint distribution having the same posterior $p_{X|Y}$ as $p_{X,Y}$, but a perturbed marginal. Indeed, taking $\alpha\geq 1-\varepsilon$, any such
$\tilde{p}_{X,Y}$ is in the TV $\varepsilon$-ball around $p_{X,Y}$, as
\begin{align}
    &d_{TV}(p_{X,Y}, \tilde{p}_{X,Y}) = \tfrac{1}{2} \iint |p_{X,Y}(x,y) - \tilde{p}_{X,Y}(x,y)| dx dy \nonumber \\ 
    &= \tfrac{1}{2} \iint | p_{X|Y}(x|y)p_Y(y) - p_{X|Y}(x|y)\tilde{p}_Y(y)| dx dy \nonumber \\
    &= \tfrac{1}{2} (1-\alpha) \iint | p_{X|Y}(x|y)p_Y(y) - p_{X|Y}(x|y)q(y)| dx dy \nonumber \\
    & \le 1-\alpha \nonumber\\
    & \le \varepsilon.
\end{align}

By our assumption that the optimal estimator is stably distribution preserving, it must satisfy $p_{\hat{X}^*} = p_X$ for any perturbation of $p_{X,Y}$ of the form \eqref{eq:perturbedPy}. Since the posterior has not changed, the optimal estimator  $p_{\hat{X}^*|Y}$ remains the same. Its marginal $\tilde{p}_{\hat{X}^*}$, however, is modified to
\begin{align}
&\tilde{p}_{\hat{X}^*}(x) = \int p_{\hat{X}^*|Y}(x|y) \tilde{p}_Y(y) dy \nonumber\\
&= \alpha \int p_{\hat{X}^*|Y}(x|y) p_Y(y) dy + (1-\alpha) \int p_{\hat{X}^*|Y}(x|y) q(y) dy \nonumber\\
&=\alpha p_X(x) + (1-\alpha) \int p_{\hat{X}^*|Y}(x|y) q(y) dy,
\end{align}
where we used the assumption that $p_{\hat{X}^*} = p_X$. Similarly, the distribution of $X$ has changed to
\begin{align}
&\tilde{p}_{X}(x) = \int p_{X|Y}(x|y) \tilde{p}_Y(y) dy \nonumber\\
&= \alpha \int p_{X|Y}(x|y) p_Y(y) dy  + (1-\alpha) \int p_{X|Y}(x|y) q(y) dy\nonumber\\
&= \alpha p_X(x) + (1-\alpha) \int p_{X|Y}(x|y) q(y) dy.
\end{align}
Thus, equality between $\tilde{p}_{\hat{X}^*}$ and $\tilde{p}_X$ is kept only if
\begin{equation}
\int p_{\hat{X}^*|Y}(x|y) q(y) dy = \int p_{X|Y}(x|y) q(y) dy.
\end{equation}
This equality can hold for \emph{every} perturbation $q$ only if $p_{\hat{X}^*|Y} = p_{X|Y}$, completing the proof.

Notice that this also proves that the optimal estimator is unique (under the stably distribution preserving assumption), as we demonstrated that only $p_{\hat{X}^*|Y} = p_{X|Y}$ minimizes the mean distortion.
\end{proof}

\begin{lemma}\label{lem:estimatorNonUnique}
    If the degradation is non-invertible, and the estimator $\hat{X}^*$ defined by $p_{\hat{X}^*|Y} = p_{X|Y}$ is an optimal estimator, then the optimal estimator is non-unique.
\end{lemma}

\begin{proof}
    Since the degradation is non-invertible, $p_{X|Y}(x|y)>0$ for all $(x,y)\in \mathcal{S}_x\times \mathcal{S}_y$, where $\mathcal{S}_x$ is a non-singleton set and $\mathcal{S}_y$ is a set that satisfies $\mathbb{P}(Y\in\mathcal{S}_y)>0$ (Definition~\ref{def:nonInvertibleDeg}). As $p_{\hat{X}^*|Y} = p_{X|Y}$, we also have that $p_{\hat{X}^*|Y}(x|y)>0$ for all $(x,y)\in \mathcal{S}_x\times \mathcal{S}_y$.
    
    Now, since $\hat{X}^*$ is an optimal estimator, $p_{\hat{X}^*|Y}$ must minimize $\int f(\hat{x},y) p_{\hat{X}|Y}(\hat{x}|y) d\hat{x}$ for each $y$ (see proof of Lemma~\ref{lem:estimatorIsPosterior}). This means that for any $y$, the conditional $p_{\hat{X}^*|Y}(\hat{x}|y)$ must assign positive probability only to $\hat{x}$ in the set of minima $\mathcal{S}_{\min}(y)=\argmin_{\hat{x}}f(\hat{x},y)$. We conclude 
    that $\mathcal{S}_{x}\subseteq \mathcal{S}_{\min}(y)$ for every $y \in \mathcal{S}_y$. This implies that any other estimator that assigns zero probability to  $\hat{x}\notin\mathcal{S}_x$ for every $y\in\mathcal{S}_y$, is also optimal.
    
    Let $\mathcal{S}_x^1,\mathcal{S}_x^2$ be non-empty disjoint sets such that $\mathcal{S}_x^1 \cup \mathcal{S}_x^2 = \mathcal{S}_x$. Now, define two estimators, such that $p_{\hat{X}_1|Y}(\hat{x}|y) > 0$ only for $\hat{x} \in \mathcal{S}_x^1$, and $p_{\hat{X}_2|Y}(\hat{x}|y) > 0$ only for $\hat{x} \in \mathcal{S}_x^2$, for every $y \in \mathcal{S}_y$. Both are optimal estimators (as they only assign positive probability to $\hat{x} \in \mathcal{S}_x$). Yet, these two estimators have conditional distributions with disjoint supports for every $y$, and thus $d_{\text{TV}}(p_{\hat{X}_1|Y}(\cdot|y),p_{\hat{X}_2|Y}(\cdot|y))>0 \quad \forall y \in \mathcal{S}_y$. Therefore, by Definition \ref{def:notUniqueEstimator}, the optimal estimator is non-unique.
\end{proof}

Now, let us assume to the contrary that $p_{X,Y}$ defines a non-invertible degradation, and that the distortion function $\Delta(\cdot,\cdot)$ is stably distribution preserving at $p_{X,Y}$. By Lemma \ref{lem:estimatorIsPosterior}, the optimal estimator $\hat{X}^*$ is uniquely defined by $p_{\hat{X}^*|Y} = p_{X|Y}$. But now according to Lemma \ref{lem:estimatorNonUnique}, since the degradation is non-invertible, the optimal estimator is non-unique, leading to a contradiction.

\section{Derivation of Example \ref{ex:scalarGaussian}}\label{ap:scalarGaussian}
Since $\hat{X} = aY = a(X+N)$, it is a zero-mean Gaussian random variable. Now, the Kullback-Leibler distance between two zero-mean normal distributions is given by
\begin{equation}
d_{\text{KL}}(p_X\|p_{\hat{X}}) = \ln \left( \frac{\sigma_{\hat{X}}}{\sigma_X} \right) + \frac{\sigma_X^2}{2\sigma_{\hat{X}}^2} - \frac{1}{2},
\end{equation}
and the MSE between $X$ and $\hat{X}$ is given by
\begin{equation}
\text{MSE}(X,\hat{X}) = E[(X-\hat{X})^2]=\sigma_X^2-2\sigma_{X\hat{X}}+\sigma_{\hat{X}}^2.
\end{equation}
Substituting $\hat{X}=aY$ and $\sigma_X^2=1$, we obtain that $\sigma_{\hat{X}}=|a|\sqrt{1 + \sigma_N^2}$ and $\sigma_{X\hat{X}}=a$, so that
\begin{align}
d_{\text{KL}}(a) &= \ln \left( |a|\sqrt{1 + \sigma_N^2} \right) + \frac{1}{2a^2(1 + \sigma_N^2)} - \frac{1}{2}, \label{eq:dKL_scalarGauss}\\
\text{MSE}(a) &= 1+a^2(1 + \sigma_N^2)-2a,\label{eq:MSE_scalarGauss}
\end{align}
and
\begin{equation}\label{eq:fAlphaScalarGaussian}
P(D) = \min_{a} d_{\text{KL}}(a)
\quad\text{s.t.}\quad  \text{MSE}(a) \le D.
\end{equation}
Notice that $d_{\text{KL}}$ is symmetric, and $\text{MSE}(|a|) \le \text{MSE}(a)$ (see Fig.~\ref{fig:dKL_MSE}). Thus, for any negative $a$, there always exists a positive $a$ with which $d_{\text{KL}}$ is the same and the MSE is not larger. Therefore, without loss of generality, we focus on the range $a\ge0$.

\begin{figure*}
	\begin{center}
		\includegraphics[width=\linewidth]{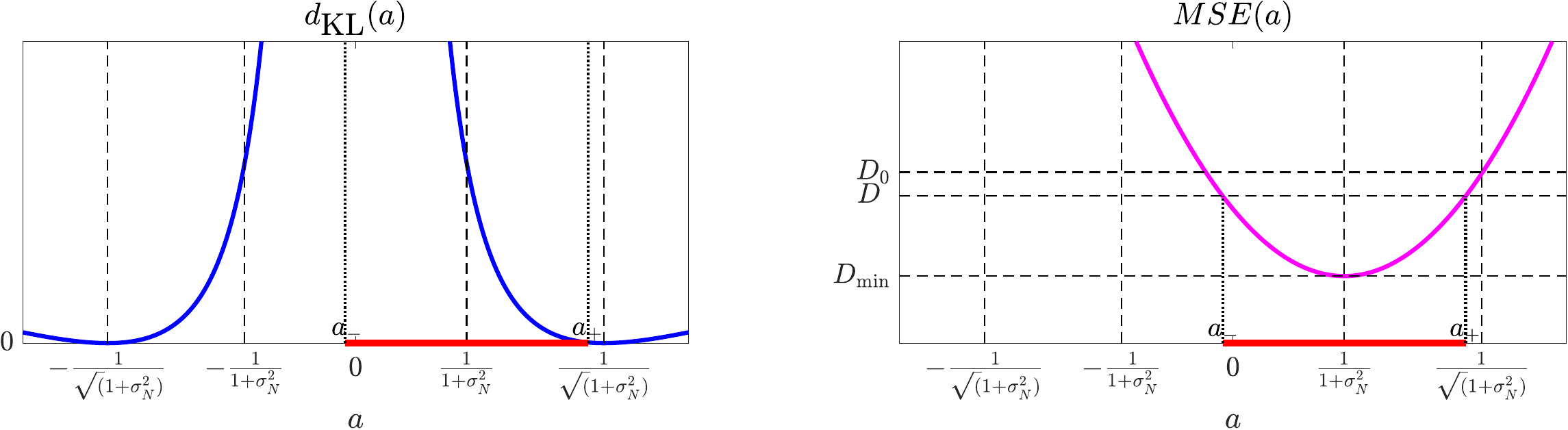}
	\end{center}
	\caption{Plots of \eqref{eq:dKL_scalarGauss} and \eqref{eq:MSE_scalarGauss}. $D$ defines the range $(a_-,a_+)$ of $a$ values complying with the MSE constraint (marked in red). The objective $d_{\text{KL}}$ is minimized over this range of possible $a$ values.}
	\label{fig:dKL_MSE}
\end{figure*}

For $D < D_{\min} = \frac{\sigma_N^2}{1+\sigma_N^2}$ the constraint set of $\text{MSE}(a)<D$ is empty, and there is no solution to \eqref{eq:fAlphaScalarGaussian}. For $D \ge D_{\min}$, the constraint is satisfied for $a_- \le a \le a_+$, where
\begin{equation}\label{eq:constraintSol}
a_{\pm}(D) = \frac{1}{(1 + \sigma_N^2)}\left(1\pm\sqrt{D(1+\sigma_N^2)-\sigma_N^2}\right).
\end{equation}
For $D=D_{\min}$, the optimal (and only possible) $a$ is
\begin{equation}
a = a_+(D_{\min}) = a_-(D_{\min}) =  \frac{1}{(1 + \sigma_N^2)}.
\end{equation}
For $D > D_{\min}$, $a_+$ monotonically increases with $D$, broadening the constraint set. The objective $d_{\text{KL}}(a)$ monotonically decreases with $a$ in the range $a \in (0,1/\sqrt{(1+\sigma_N^2)})$ (see Fig.~\ref{fig:dKL_MSE} and the mathematical justification below). Thus, for $D_{\min} < D \le D_0$, the optimal $a$ is always the largest possible $a$, which is $a=a_+(D)$, where $D_0$ is defined by $a_+(D_0) = 1/\sqrt{(1+\sigma_N^2)}$ (see Fig.~\ref{fig:dKL_MSE}).
For $D>D_0$, the optimal $a$ is $a=1/\sqrt{(1+\sigma_N^2)}$, which achieves the global minimum $d_{\text{KL}}(a)=0$. The closed form solution is therefore given by
\begin{equation}
P(D) =\begin{cases}
\begin{aligned}
&d_{\text{KL}}(a_+(D)) \quad \quad &D_{\min} \le D < D_0\\
&0 &D_0 \le D
\end{aligned}
\end{cases}
\end{equation}

To justify the monotonicity of $d_{\text{KL}}(a)$ in the range $a \in (0,1/\sqrt{(1+\sigma_N^2)})$, notice that for $a > 0$,
\begin{equation}
\frac{d}{da} d_{\text{KL}}(a) = \frac{1}{a} - \frac{1}{(1 + \sigma_N^2)} \frac{1}{a^3},
\end{equation}
which is negative for $a \in (0,1/\sqrt{(1+\sigma_N^2)})$.

\section{Proof of Theorem \ref{lem:convexity}}\label{ap:convexityProof}

\begin{figure}
	\begin{center}
		\includegraphics[width=\linewidth]{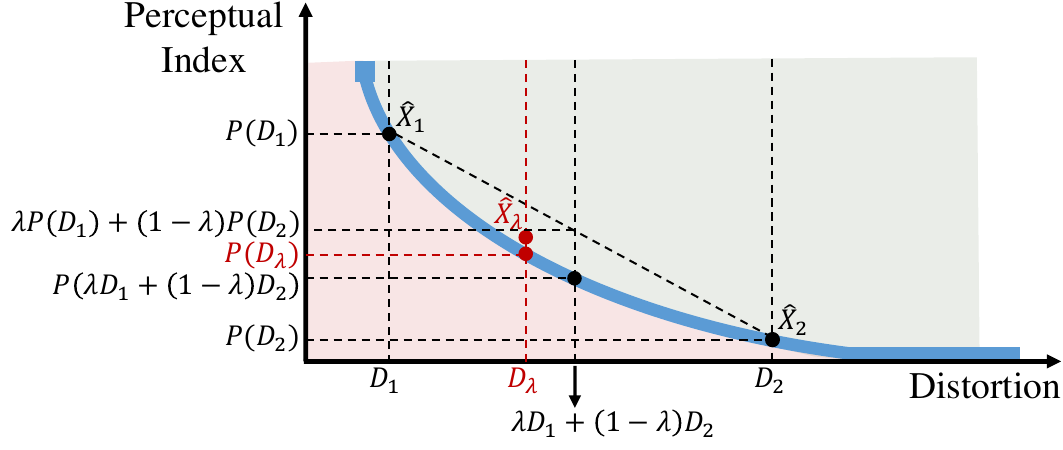}
	\end{center}
	\caption{Illustration of the proof of Theorem \ref{lem:convexity}.}
	\label{fig:theoremConvexity}
\end{figure}

The proof of Theorem \ref{lem:convexity} follows closely that of the rate-distortion theorem from information theory \cite{cover2012elements}. The value $P(D)$ is the minimal distance $d(p_X,p_{\hat{X}})$ over a constraint set whose size does not decrease with~$D$. This implies that the function $P(D)$ is non-increasing in $D$. Now, to prove the convexity of $P(D)$, we will show that
\begin{equation}\label{eq:proofObjective}
\lambda P(D_{1})+(1-\lambda)P(D_2) \ge  P(\lambda D_{1}+(1-\lambda)D_{2}),
\end{equation}
for all $\lambda \in [0,1]$ (see Fig.~\ref{fig:theoremConvexity}). First, by definition, the left hand side of~\eqref{eq:proofObjective} can be written as
\begin{equation}\label{eq:pr1}
\lambda d(p_{X}, p_{\hat{X}_{1}}) + (1-\lambda)d(p_{X}, p_{\hat{X}_{2}}),
\end{equation}
where $\hat{X}_1$ and $\hat{X}_2$ are the estimators defined by
\begin{align}\label{eq:defX1}
p_{\hat{X}_1 \vert Y} &= \argmin_{p_{\hat{X} \vert Y}}\limits d(p_X,p_{\hat{X}})\;\;
\text{s.t.}\;\;  \EE{\Delta(X,\hat{X})} \le D_1,  \\
p_{\hat{X}_2 \vert Y} &= \argmin_{p_{\hat{X} \vert Y}}\limits d(p_X,p_{\hat{X}}) \;\;
\text{s.t.}\;\;  \EE{\Delta(X,\hat{X})} \le D_2.
\end{align}
Since $d(\cdot,\cdot)$ is convex in its second argument,
\begin{equation}\label{eq:pr2}
\lambda d(p_{X}, p_{\hat{X}_{1}}) + (1-\lambda)d(p_{X}, p_{\hat{X}_{2}}) \ge d(p_{X}, p_{\hat{X}_{\lambda}}),
\end{equation}
where $\hat{X}_{\lambda}$ is defined by
\begin{equation}\label{eq:define_XhatLambda}
p_{\hat{X}_{\lambda}\vert Y}=\lambda p_{\hat{X}_1\vert Y}+\left(1-\lambda \right)p_{\hat{X}_2\vert Y}.
\end{equation}
Denoting $D_\lambda = \E[\Delta(X,\hat{X}_\lambda)]$, we have that
\begin{align}\label{eq:pr3}
d(p_{X}, p_{\hat{X}_{\lambda}}) &\ge \min_{p_{\hat{X} \vert Y}} \left\{ d(p_X,p_{\hat{X}}) \, : \, \E[\Delta(X,\hat{X})] \le D_\lambda \right\} \nonumber \\
&= P(D_{\lambda}),
\end{align}
because $\hat{X}_{\lambda}$ is in the constraint set. Below, we show that
\begin{equation}\label{eq:pr5}
D_{\lambda} \le \lambda D_1 + (1-\lambda)D_2.
\end{equation}
Therefore, since $P(D)$ is non-increasing in $D$, we have that
\begin{equation}\label{eq:pr4}
P(D_{\lambda}) \ge P(\lambda D_{1}+(1-\lambda)D_2).
\end{equation}
Combining \eqref{eq:pr1}, \eqref{eq:pr2}, \eqref{eq:pr3} and \eqref{eq:pr4} proves \eqref{eq:proofObjective}, thus demonstrating that $P(D)$ is convex.

To justify \eqref{eq:pr5}, note that
\begin{align}
D_\lambda &= \E\left[\Delta(X,\hat{X_\lambda})\right] \nonumber \\
&= \E \left[ \E\left[\Delta(X,\hat{X_\lambda}) \vert Y \right]  \right] \nonumber \\
&= \E \left[ \lambda \E\left[\Delta(X,\hat{X_1}) \vert Y \right] + (1-\lambda) \E\left[ \Delta(X,\hat{X_2}) \vert Y \right]  \right] \nonumber \\
&= \lambda\EE{\Delta(X,\hat{X}_1)}+\left(1-\lambda\right)\EE{\Delta(X,\hat{X}_2)} \nonumber\\
&\le \lambda D_1 + (1-\lambda)D_2,
\end{align}
where the second and fourth transitions are according to the law of total expectation and the third transition is justified by
\begin{align}
p(x,\hat{x}_\lambda | y) &= p(\hat{x}_\lambda | x,y) p(x | y) \nonumber\\
&= p(\hat{x}_\lambda | y) p(x | y) \nonumber\\
&= (\lambda p(\hat{x}_1 | y) + (1-\lambda )p(\hat{x}_2 | y)) p(x | y) \nonumber \\
&= \lambda p(\hat{x}_1 | y) p(x | y) + (1-\lambda )p(\hat{x}_2 | y)) p(x | y) \nonumber\\
&= \lambda p(x,\hat{x}_1 | y) + (1-\lambda )p(x, \hat{x}_2 | y)).
\end{align}
Here we used \eqref{eq:define_XhatLambda} and the fact that $X$ and $\hat{X}_\lambda$ are independent given $Y$, and similarly for the pairs $(X,\hat{X}_1)$ and $(X,\hat{X}_2)$.

\section{Proof of Theorem \ref{thm:bound}}\label{ap:boundProof}
The estimator $\hat{X}$ of \eqref{eq:xpost} attains perfect perceptual quality since
\begin{align}\label{eq:sameDist}
p_{\hat{X}}(x) &= \int p_{\hat{X}|Y}(x|y) p_Y(y) dy \nonumber \\
&= \int p_{X|Y}(x|y) p_Y(y) dy \nonumber\\
&= p_X(x).
\end{align}
Furthermore, note that
\begin{align}\label{eq:XtEhatX}
\E[ X^T\hat{X}] &= \E[ \E[X^T\hat{X}|Y] ] \nonumber\\
&= \E[ \E[X|Y]^T\E[\hat{X} |Y] ] \nonumber \\
&= \E[\| \E[X|Y]\|^2],
\end{align}
and
\begin{align}\label{eq:EhatX}
\E[\|\hat{X}\|^2] = \E[ \E[\|\hat{X}\|^2|Y] = \E[ \E[\|X\|^2|Y] = \E[\|X\|^2],
\end{align}
where we used the law of total expectation and the fact that given $Y$, $X$ and $\hat{X}$ are independent and identically distributed. The MSE of $\hat{X}$ is therefore
\begin{align}
\E[\| X-\hat{X}\|^2] &= \E[\|X\|^2] -2 \E[ X^T\hat{X}] + \E[\|\hat{X}\|^2] \nonumber \\
&= 2(\E[\|X\|^2] - \E[ \|\E[X|Y]\|^2]) \nonumber \\
&= 2\,\E[\| X - \E[X|Y] \|^2] \nonumber \\
&= 2\,\E[\| X - \hat{X}_{\text{MMSE}} \|^2],
\end{align}
where the second equality is due to \eqref{eq:XtEhatX} and \eqref{eq:EhatX}, and the third equality is due to the orthogonality principle. We thus established that $\hat{X}$ is a distribution preserving estimator whose MSE is precisely twice the MSE of the MMSE estimator. This implies that
\begin{equation}
D_{\max} \le \E[\| X-\hat{X} \|^2]  =  2 D_{\min},
\end{equation}
completing the proof.

\section{WGAN architecture and training details (Sec. \ref{sec:WGAN})}\label{ap:WGANdetails}
\begin{table*}
	\centering
	\caption{Generator and discriminator architecture. FC is a fully-connected layer, BN is a batch-norm layer, and l-ReLU is a leaky-ReLU layer.}\label{tab:WGAN}
	\begin{tabular}[t]{|c|c|}
		\hline
		\multicolumn{2}{c}{Discriminator} \\
		\hline
		Size & Layer\\
		\hline \hline
		$28 \times 28 \times 1$ & Input  \\
		\hline
		$12 \times 12 \times 32$ & Conv (stride=2), l-ReLU (slope=$0.2$) \\
		\hline
		$4 \times 4 \times 64$ & Conv (stride=2), l-ReLU (slope=$0.2$) \\
		\hline
		$1024$ & Flatten\\
		\hline
		$1$ & FC\\
		\hline
		$1$ & Output\\
		\hline
	\end{tabular}
	\quad\quad\quad
	\begin{tabular}[t]{|c|c|}
		\hline
		\multicolumn{2}{c}{Generator} \\
		\hline
		Size & Layer\\
		\hline \hline
		$28 \times 28 \times 1$ & Input  \\
		\hline
		$784$ & Flatten \\
		\hline
		$4 \times 4 \times 128$ & FC, unflatten, BN, ReLU \\
		\hline
		$7 \times 7 \times 64$ & transposed-Conv (stride=$2$), BN, ReLU \\
		\hline
		$14 \times 14 \times 32$ & transposed-Conv (stride=$2$), BN, ReLU \\
		\hline
		$28 \times 28 \times 1$ & transposed-Conv (stride=$2$), sigmoid\\
		\hline
		$28 \times 28 \times 1$ & Output\\
		\hline
	\end{tabular}	
\end{table*}

The architecture of the WGAN trained for denoising the MNIST images is detailed in Table \ref{tab:WGAN}. The training algorithm and adversarial losses are as proposed in \cite{gulrajani2017improved}. The generator loss was modified to include a content loss term, \ie $\ell_{\text{gen}} = \ell_{\text{MSE}} + \lambda \, \ell_{\text{adv}}$, where $\ell_{\text{MSE}}$ is the standard MSE loss. For each $\lambda$ the WGAN was trained for 35 epochs, with a batch size of 64 images. The ADAM optimizer \cite{kingma2017adam} was used, with $\beta_1=0.5, \beta_2=0.9$. The generator\slash discriminator initial learning rate is $10^{-3}\slash 10^{-4}$ respectively, where learning rate of both decreases by half every 10 epochs. The filter size of the discriminator convolutional layers is $5\times 5$, and these are performed without padding. The filter size in the generator transposed-convolutional layers is $5\times 5 \slash 4\times 4$, and these are performed with $2\slash 1$ pixel padding for the first\slash second and third transposed-convolutional layers, respectively. The stride of each convolutional layer and the slope for the leaky-ReLU layers appear in Table \ref{tab:WGAN}. Note that the perception-distortion curve in Fig.~\ref{fig:WGAN} is generated by training on single digit images, which in general may deviate from the perception-distortion curve of  whole images containing i.i.d. sub-blocks of digits.

\section{Super-resolution evaluation details (Sec. \ref{sec:practicalMethod}) and additional comparisons}\label{sec:SR_details}
The no-reference (NR) and full-reference (FR) methods BRISQUE, BLIINDS-II, NIQE, SSIM, MS-SSIM, IFC and VIF were obtained from the LIVE laboratory website\footnote{\url{http://live.ece.utexas.edu/research/Quality/index.htm}}, the NR method of Ma \etal was obtained from the project webpage\footnote{\url{https://github.com/chaoma99/sr-metric}}, and the pretrained VGG-19 network was obtained through the PyTorch torchvision package\footnote{\url{http://pytorch.org/docs/master/torchvision/index.html}}. The low-resolution images were obtained by factor 4 downsampling with a bicubic kernel. The super-resolution results on the BSD100 dataset of the SRGAN and SRResNet variants were obtained online\footnote{\url{https://twitter.box.com/s/lcue6vlrd01ljkdtdkhmfvk7vtjhetog}}, and the results of EDSR, Deng, Johnson \etal~and Mechrez \etal~were kindly provided by the authors. The algorithms for testing the other SR methods were obtained online: A+\footnote{\url{http://www.vision.ee.ethz.ch/~timofter/ACCV2014_ID820_SUPPLEMENTARY/}}, SRCNN\footnote{\url{http://mmlab.ie.cuhk.edu.hk/projects/SRCNN.html}}, SelfEx\footnote{\url{https://github.com/jbhuang0604/SelfExSR}},  VDSR\footnote{\url{http://cv.snu.ac.kr/research/VDSR/}}, LapSRN\footnote{\url{https://github.com/phoenix104104/LapSRN}}, Bae \etal\footnote{\url{https://github.com/iorism/CNN}} and ENet\footnote{\url{https://webdav.tue.mpg.de/pixel/enhancenet/}}. All NR and FR metrics and all SR algorithms were used with the default parameters and models. In the paper, we reported comparisons on the y-channel (except for the $\text{VGG}_{2,2}$ measure). In the supplementary material, we report results with additional NR metrics on the y-channel, as well as results on color images. When comparing color images, for SR algorithms which treat the y-channel alone, the Cb and Cr channels are upsampled by bicubic interpolation.

The general pattern appearing in Fig.~\ref{fig:noRefMethods1} will appear for any NR method which accurately predicts the perceptual quality of images. We show here three additional popular NR methods: BRISQUE \cite{mittal2012no}, BLIINDS-II \cite{saad2012blind} and the recent measure by Ma \etal \cite{ma2017learning} in Figs.~\ref{fig:noRefMethods3}, \ref{fig:noRefMethods4}, \ref{fig:noRefMethods5}, where the same conclusions as for NIQE \cite{mittal2013making} (see Sec. \ref{sec:practicalMethod}) are apparent. The same pattern appears for RGB images as well, as shown in Figs.~\ref{fig:noRefMethods6}, \ref{fig:noRefMethods7}. Note that the perceptual quality of Johnson \etal and SRResNet-VGG$_{2,2}$ is inconsistent between NR metrics, likely due to varying sensitivity to the cross-hatch pattern artifacts which are present in these method's outputs. For this reason, Johnson \etal does not appear in the NIQE plots, as its NIQE score is $13.55$ (far off the plots).

\begin{figure*}
	\begin{center}
		\includegraphics[width=\linewidth]{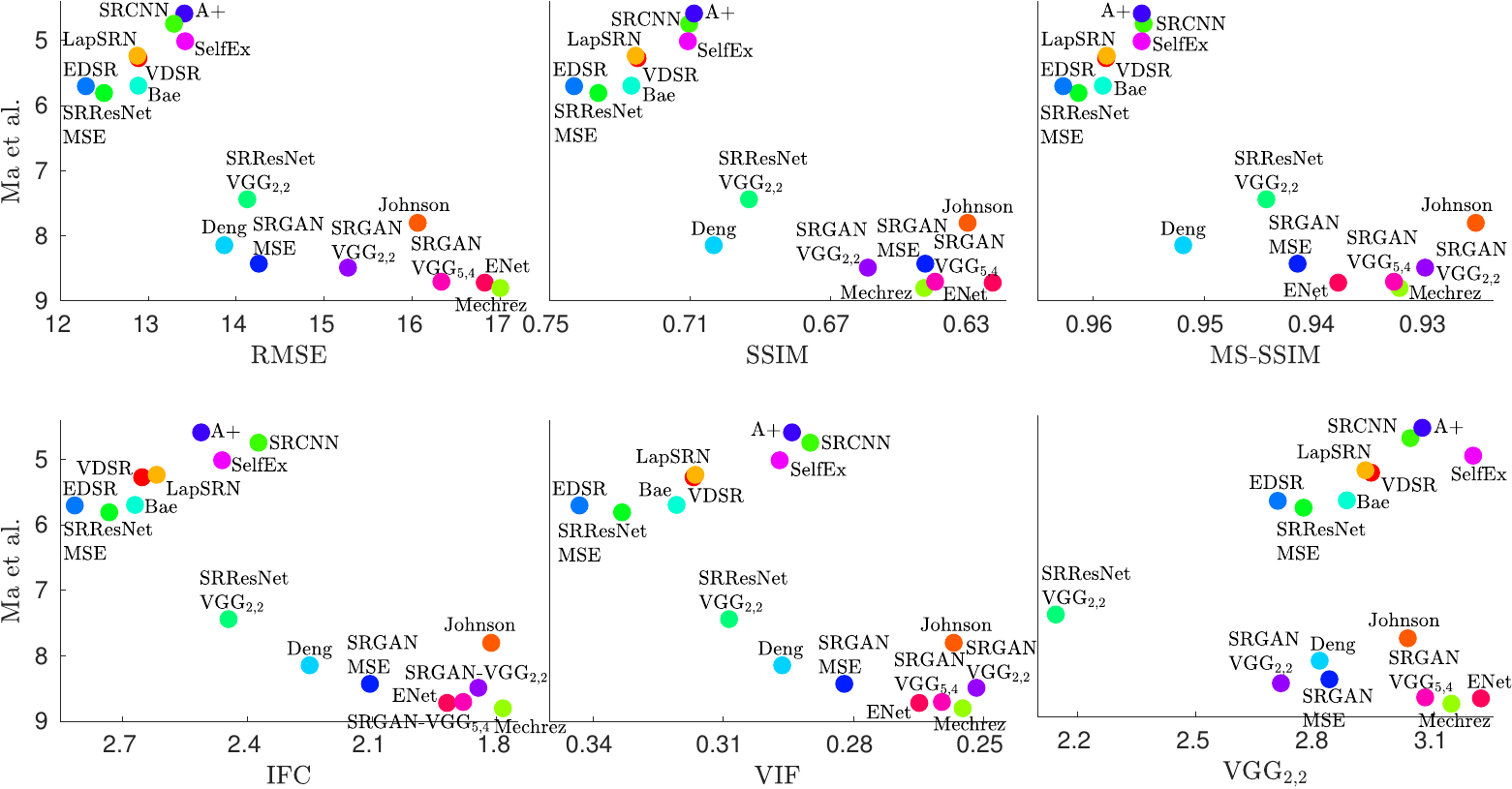}
	\end{center}
	\caption{Plot of $15$ algorithms on the perception-distortion plane, where perception is measured by the NR metric by Ma \etal \cite{ma2017learning}, and distortion is measured by the common full-reference metrics RMSE, SSIM, MS-SSIM, IFC, VIF and VGG$_{2,2}$. All metrics were calculated on the \textbf{y-channel} alone.}
	\label{fig:noRefMethods3}
\end{figure*}

\begin{figure*}
	\begin{center}
		\includegraphics[width=\linewidth]{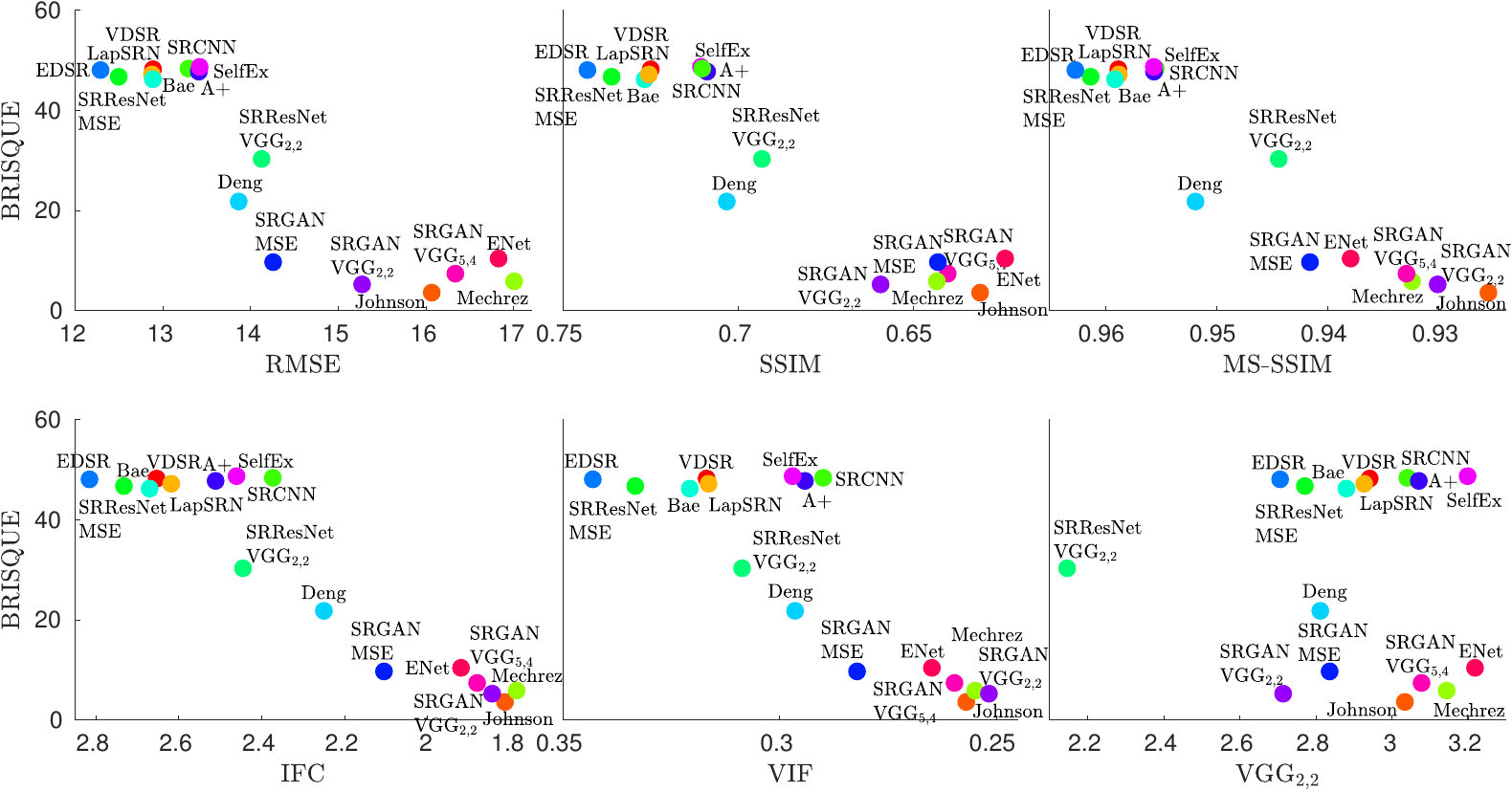}
	\end{center}
	\caption{Plot of $16$ algorithms on the perception-distortion plane, where perception is measured by the NR metric BRISQUE, and distortion is measured by the common full-reference metrics RMSE, SSIM, MS-SSIM, IFC, VIF and VGG$_{2,2}$. All metrics were calculated on the \textbf{y-channel} alone.}
	\label{fig:noRefMethods4}
\end{figure*}

\begin{figure*}
	\begin{center}
		\includegraphics[width=\linewidth]{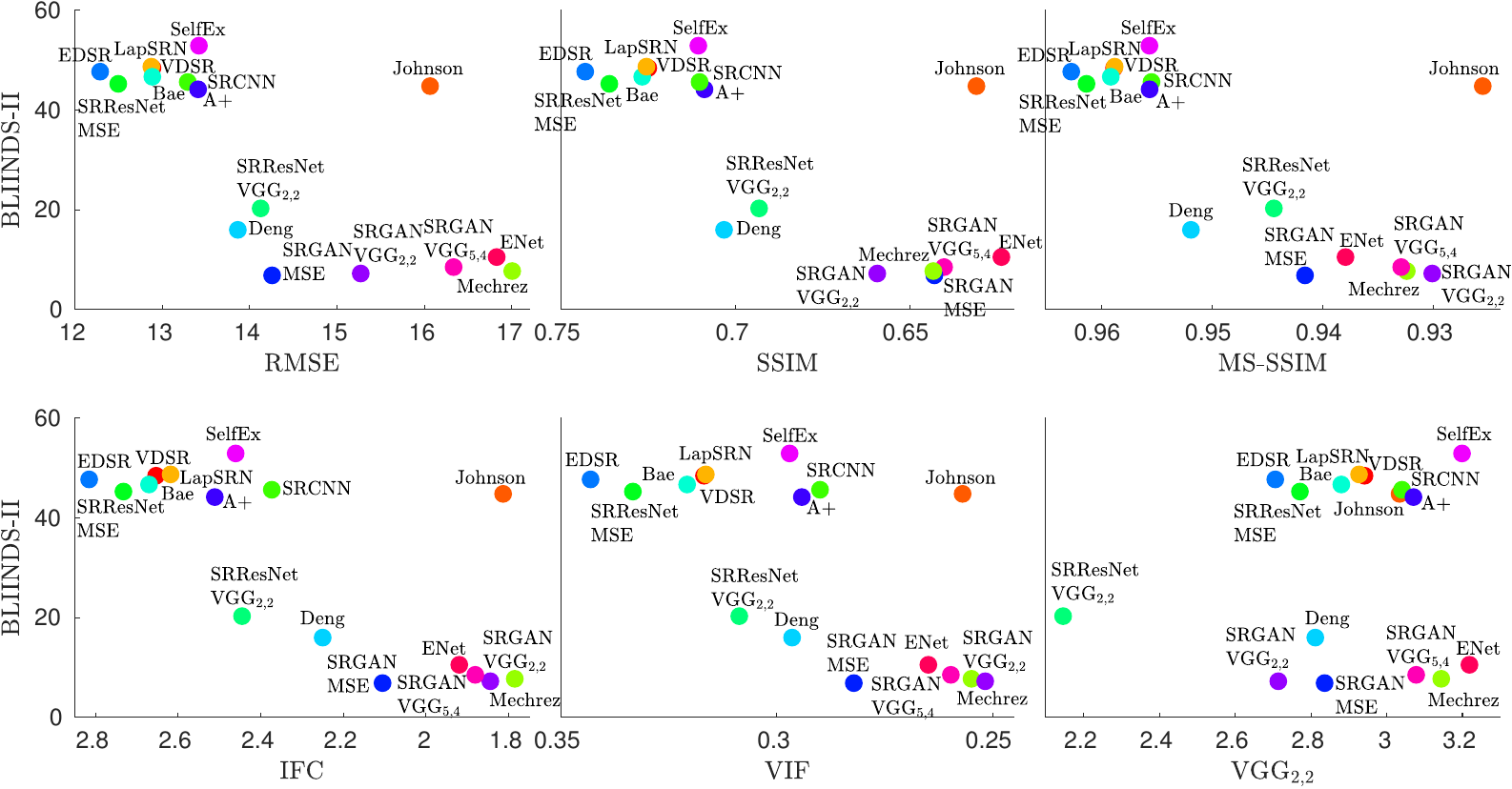}
	\end{center}
	\caption{Plot of $16$ algorithms on the perception-distortion plane, where perception is measured by the NR metric BLIINDS-II, and distortion is measured by the common full-reference metrics RMSE, SSIM, MS-SSIM, IFC, VIF and VGG$_{2,2}$. All metrics were calculated on the \textbf{y-channel} alone.}
	\label{fig:noRefMethods5}
\end{figure*}

\begin{figure*}
	\begin{center}
		\includegraphics[width=\linewidth]{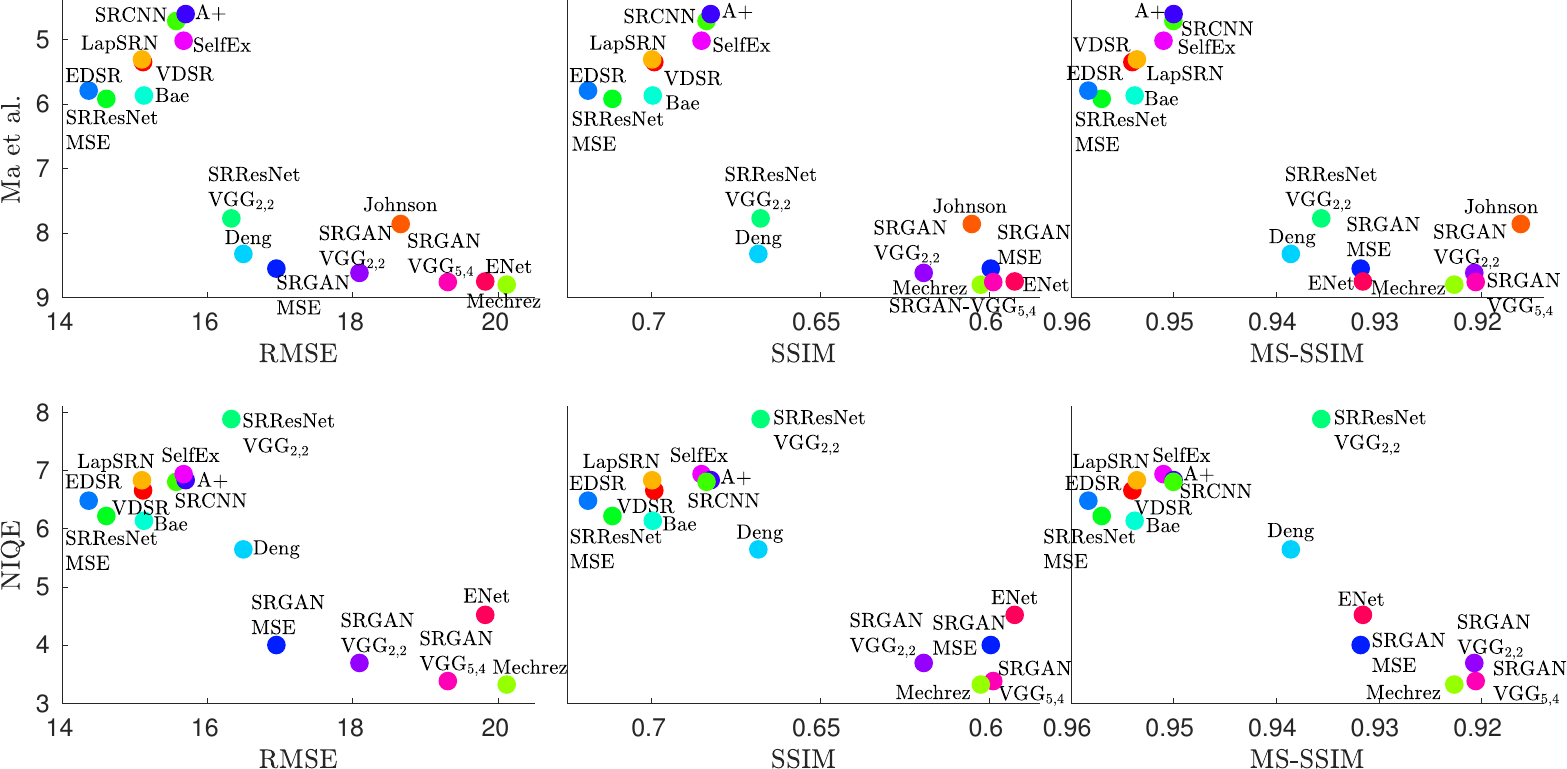}
	\end{center}
	\caption{Plot of $16$ algorithms on the perception-distortion plane. Perception is measured by the the NR metrics of Ma \etal~and NIQE, and distortion is measured by the common full-reference metrics RMSE, SSIM and MS-SSIM. All metrics were calculated on \textbf{three channel RGB} images.}
	\label{fig:noRefMethods6}
\end{figure*}

\begin{figure*}
	\begin{center}
		\includegraphics[width=\linewidth]{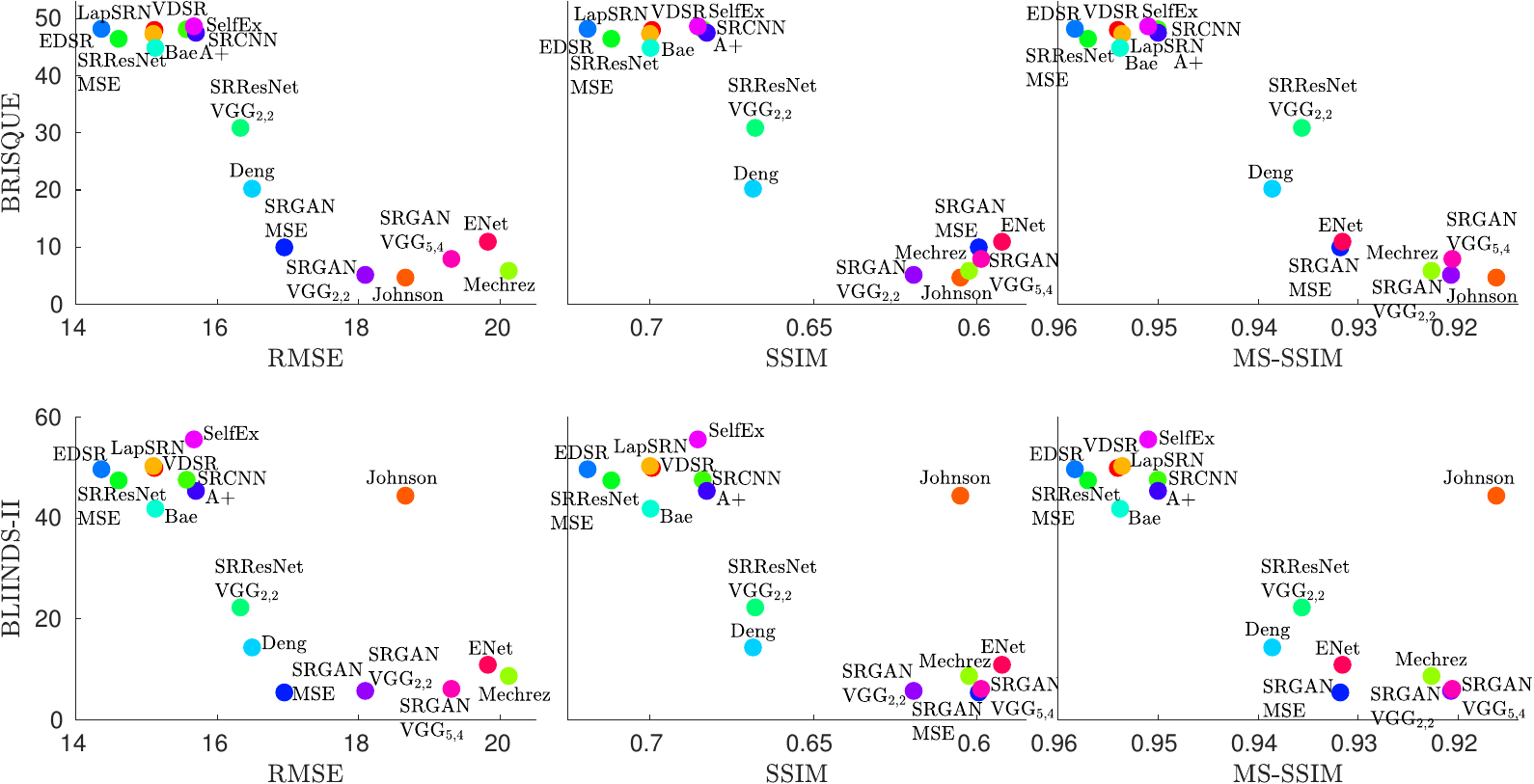}
	\end{center}
	\caption{Plot of $16$ algorithms on the perception-distortion plane. Perception is measured by the the NR metrics BRISQUE and BLIINDS-II, and distortion is measured by the common full-reference metrics RMSE, SSIM and MS-SSIM. All metrics were calculated on \textbf{three channel RGB} images.}
	\label{fig:noRefMethods7}
\end{figure*}

\end{document}